\documentclass[twocolumn]{svjour3_mod} 

\pdfpageattr {/Group << /S /Transparency /I true /CS /DeviceRGB>>}

\title{Exact Bayesian inference for off-line change-point detection in tree-structured graphical models
}
\author{L. Schwaller \and S. Robin
}

\institute{
L. Schwaller \and S. Robin
\at UMR MIA-Paris, AgroParisTech, INRA,\\ Universit\'e Paris-Saclay, 75005, Paris, France\\
\email {loic.schwaller@agroparistech.fr}
}

\date{}

\usepackage[dvipsnames]{xcolor}
\usepackage{etex}
\usepackage{mathrsfs}
\usepackage{verbatim}
\usepackage{amsmath}
\usepackage{amssymb}
\usepackage{amsfonts,amssymb,verbatim}
\usepackage{mathrsfs}
\usepackage{graphicx}
\usepackage{tikz}
\usetikzlibrary{arrows}
\usepackage{enumitem}
\usepackage{placeins}
\usepackage{float}
\usepackage{nicefrac}
\usepackage{array}
\usepackage{stmaryrd}
\usepackage{braket}
\usepackage{apacite}
\usepackage{natbib}
\usepackage{enumitem}
\usepackage{caption}
\usepackage{subfig}
\usepackage{dblfloatfix}
\usepackage{multirow}

\xdefinecolor{steelblue}{named}{NavyBlue}

\xdefinecolor{tree}{RGB}{162,192,217}
\xdefinecolor{ER2p}{RGB}{246,173,160}
\xdefinecolor{ER4p}{RGB}{230,167,156}
\xdefinecolor{ER8p}{RGB}{197,154,146}

\newcommand{\blue}[1]{{\leavevmode\color{black}{#1}}}

\newcommand{\LS}[2]{\textcolor{gray}{#1}\textcolor{black}{#2}}

\def\T{{\mathcal T}}
\def\M{{\mathcal M}}
\def\TT{{\mathbf T}}
\def\p{{\mathbf p}}

\def\E{{\mathcal E}}

\def\XX{{\mathcal X}}

\DeclareMathOperator*{\argmax}{arg\,max}

\newcommand{\defeq}{\mathrel{\vcenter{\baselineskip0.5ex \lineskiplimit0pt
                     \hbox{\scriptsize.}\hbox{\scriptsize.}}}%
                     =}

\smartqed

\AtBeginDocument{%
  \setlength{\oddsidemargin}{\dimexpr(\paperwidth-\textwidth)/2-1in}%
  \setlength{\evensidemargin}{\oddsidemargin}%
  \setlength{\topmargin}{%
    \dimexpr(\paperheight-\textheight)/2-\headheight-\headsep-1in}%
}

\begin{document}

\maketitle

\begin{abstract}
We consider the problem of change-point detection in multivariate time-series. The multivariate distribution of the observations is supposed to follow a graphical model, whose graph and parameters are affected by abrupt changes throughout time. We demonstrate that it is possible to perform exact Bayesian inference whenever one considers a simple class of undirected graphs called spanning trees as possible structures. We are then able to integrate on the graph and segmentation spaces at the same time by combining classical dynamic programming with algebraic results pertaining to spanning trees. In particular, we show that quantities such as posterior distributions for change-points or posterior edge probabilities over time can efficiently be obtained. We illustrate our results on both synthetic and experimental data arising from biology and neuroscience. 
\end{abstract}


\section{Introduction}

We are interested in time-series data where several variables are observed throughout time. An assumption often made in multivariate settings is that there exists an underlying network describing the dependences between the different variables. When modelling time-series data, one is faced with a choice: shall this network be considered stationary or not? \blue{Taking the example of genomic data, it might for instance be unrealistic to consider that the network describing how a pool of genes regulate each other remains identical throughout time. This network might slowly evolve, or undergo abrupt changes leading to the initialisation of new morphological development stages in the organism of interest. Here, we focus our interest on the second scenario.}


\blue{The inference of the dependence structure ruling a multivariate time-series was first performed under the assumption that this structure was stationary (\textit{e.g.} \citep{Friedman1998,Murphy99}).} Non-stationarity has then been addressed in a variety of ways. Classical Dynamic Bayesian Networks (DBNs) can for instance be adapted to allow the directed graph (or Bayesian Network) describing the interactions between two consecutive time-points to change, leading to so-called switching DBNs \citep{Robinson2010,Lebre2010,Grzegorczyk2011}. Some models alternatively suppose that the heterogeneity is the result of parameters changing smoothly with time \citep{Zhou2010,Kolar2010}. This is especially appropriate for Gaussian graphical models where the graph structure can directly be read in the non-zero terms of the precision (or inverse-covariance) matrix, therefore enabling smooth transitions within the otherwise discrete space of graphs. Hidden Markov Models (HMM) have also been used to account for heterogeneity in multivariate time-series \citep{Fox2009,Barber2010}. In the aforementioned models, the inference can rarely be performed exactly, and often relies on sampling techniques such as Markov Chain Monte Carlo (MCMC).

\begin{figure*}[b]
\centering
\includegraphics[width=0.8\textwidth]{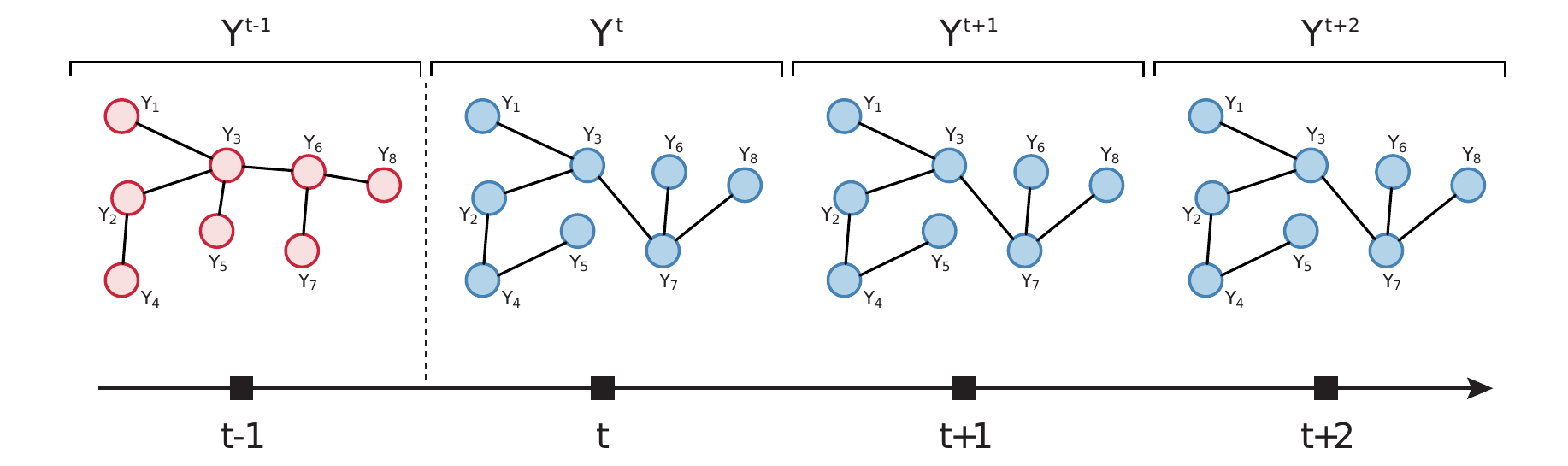}
\caption{Illustration of the change-point detection problem in the tree structure of a graphical model.}
\label{fig:illustration_chgtpt}
\end{figure*}

The model that we consider here belongs to the class of product partition models (PPM) \citep{Barry1992}. We assume that the observed data $\{y^{t}\}_{t=1,...,N}$ are a realisation of a process $\{Y^{t}\}_{t=1,...,N}$  where, for $1 \leqslant t \leqslant T$, $Y^{t}$ is a random vector with dimension $p \geqslant 2$. If $m$ is a segmentation of $\{1, \dots, T\}$ with change-points $1 = \tau_0 < \tau_1 < \dots < \tau_{K-1} < \tau_K = N$, the model has the general form
\begin{align*}
Y_t &\sim \pi(G_r,\theta_r), & \textrm{if}~ t\in r ~ \textrm{and} ~ r = \llbracket\tau_i;\tau_{i+1}\llbracket,
\end{align*} 
where $G_r$ and $\theta_r$ respectively stand for the graph describing the dependence structure and the distribution parameters on segment $r$. The parameters $(G_r,\theta_r)$  are assumed to be independent between segments. \blue{ This model is illustrated in Figure \ref{fig:illustration_chgtpt}.}

We are interested in retrieving all change-points at the same time, therefore performing off-line detection. It has been shown that both off-line \citep{Fearnhead2006,Rigaill2012} and on-line detection \citep{Fearnhead2007,Caron2012} of change-points can be performed exactly and efficiently in this model thanks to dynamic programming. \cite{Xuan2007} explicitly consider this framework in a multivariate Gaussian setting. They estimate a set of possible structures for their model by performing regularized estimation of the precision matrix on arbitrary overlapping time segments. This graph family is then taken as a starting point in an iterative procedure where the segmentation and the graph family are sequentially updated to get the best segmentation and graph series.

\paragraph{Our contribution}From a Bayesian point of view, the problem at hand raises an interesting and quite typical problem as both continuous and discrete parameter are involved in the model. Indeed, the location and scale parameters or, more specifically, the means and (conditional) covariances associated with each segments are continuous but the location of the change-points and the structure of the graphical model within each segments are not. Denoting $\theta$ the set of continuous parameters, $Q$ the set of discrete parameters and $y$ the observed data, Bayesian inference will typically rely on integrals such as the marginal likelihood of the data, that is
$$
p(y) = \sum_{Q \in \mathcal{Q}} p(Q) \int_{\theta \in \Theta} p(y|\theta, Q) p(\theta|Q) d\theta. 
$$
In many situations, the use of conjugate priors allows us to compute the integral with respect to $\theta$ in an exact manner. Still, the summation over all possible values for the discrete parameter $Q$ is often intractable due to combinatorial complexity. One aim of this article is to remind that the algebraic properties of the space $\mathcal{Q}$ can sometimes help to actually achieve this summation in an exact manner, so that a fully exact Bayesian inference can be carried out.\\

We show that exact and efficient Bayesian inference can be performed in a multivariate product partition model within the class of undirected graphs called spanning trees. These structures are connected graphs, with no cycles \blue{(see Figure \ref{fig:illustration_chgtpt} for examples)}. When $p$ nodes are considered, we are left with $p^{p-2}$ possible spanning trees, but exact inference remains tractable by using algebraic properties pertaining to this set of graphs. On each independent temporal segment, we place ourselves in the framework developed by \cite{Schwaller2015}, in which the likelihood of a segment $\llbracket s; t \rrbracket$, defined by
\begin{align*}
p(y^{\llbracket s; t \rrbracket}) \defeq \sum_{T\in \T} \int p(y^{\llbracket s; t \rrbracket} | \theta,T)p(\theta|T)d\theta,
\end{align*}
where $\T$ stands for the set of spanning trees, can be computed efficiently. We provide explicit and  exact formulas for quantities of interest such as the posterior distribution of change-points or posterior edge probabilities over time. We also provide a way to assess whether the status of an edge (or of the whole graph) remains identical throughout the time-series or not when the partition is given. 

\paragraph{Outline}
In Section \ref{sec:background}, we provide some background on graphical models and product partition models. In particular, we give a more detailed presentation of the results of \cite{Rigaill2012} on dynamic programming used for change-point detection problems. We also introduce tree-structured graphical models. The model and its properties are presented in Section \ref{sec:model}. Section \ref{sec:QoI} enumerates a list of quantities of interest that can be computed in this model, while Section \ref{sec:edge_status} deals with edge and graph status comparison, when the segmentation is known. Sections \ref{sec:simulation} and \ref{sec:appli} respectively present the simulation study and the applications to both  biological and neuroscience data.

\section{Background}
\label{sec:background}

In this section we introduce two models involving a discrete parameter, for which exact integration over this parameter is possible.

\subsection{Product Partition Models}
\label{subsec:ppm}

Let $Y=\{Y^{t}\}_{t=1,...,N}$ be an independent random process and let $y$ be a realisation of this process. For any time interval $r$, we let $Y^r \defeq \{Y^{t}\}_{t\in r}$ denote the observations for $t\in r$. PPMs as described in \citep{Barry1992} work under the assumption that the observations can be divided in independent adjacent segments. Thus, if $m$ is a partition of $\llbracket 1 ; N \rrbracket$, the likelihood of $y$ conditioned on $m$ can be written as 
\begin{align*}
p(y|m) & = \prod_{r\in m} p(y^r|r), \\
p(y^r|r) & = \int \left(\prod_{t \in r} p(y^t|\theta_r) \right)p(\theta_r)d\theta_r,
\end{align*}
where $\theta_r$ is a set of parameters giving the distribution of $Y^t$ for $t \in r$. For the sake of clarity, we let $p(y^r)$ denote $p(y^r|r)$ in the following.

For $K \geqslant 1$, we let $\mathcal{M}_K$ denote the set made of the partitions of $\llbracket 1 ; N \rrbracket$ into $K$ segments. \LS{}{The cardinality of this set is $\binom{N-1}{K-1}$.} More generally, we let $\mathcal{M}_K(\llbracket s;t \llbracket)$ denote the partitions of any interval $\llbracket s;t \llbracket$ into $K$ segments. In order to get the marginal likelihood of $y$ conditionally on $K$, one has to integrate out both $m$ and $\theta = \{\theta_r\}_{r\in m}$:
\begin{align*}
p(y|K) &= \sum_{m \in \mathcal{M}_K} p(m)\prod_{r\in m}p(y^r) \\
& = \sum_{m \in \mathcal{M}_K} p(m)\prod_{r\in m}\int \left(\prod_{t \in r} p(y^t|\theta_r) \right)p(\theta_r)d\theta_r.
\end{align*}
If the distribution of $m$, conditional on $K$, factorises over the segments \LS{}{with an expression of the form}
\begin{align}
p(m|K) = \frac{1}{ C_K(a)}\prod_{r\in m} a_r, \label{eq:prior_segmentation}
\end{align} 
where $a_r$ are non-negative weights assigned to all segments and $C_K(a) = \sum_{m\in \M_K}\prod_{r\in m} a_r$ is a normalising constant, these integrations can be performed separately. \cite{Rigaill2012} introduced a matrix containing the weighted likelihood of all possible segments, whose general term is given by
\begin{align}
A_{s,t} = \left\lbrace\begin{array}{ll}
a_{\llbracket s;t\llbracket}\cdot p(y^{\llbracket s;t\llbracket}) & \textrm{if}~1\leqslant s<t \leqslant N+1, \\ 
0 & \textrm{otherwise}.
\end{array} \right. \label{eq:A}
\end{align}
This matrix can be used in an algorithm designed according to a dynamic programming principle to perform the integration on $\mathcal{M}_K$ efficiently. 
\begin{proposition}[\cite{Rigaill2012}]
\label{prop:rigaill}
\begin{align*}
[A^K]_{s,t} = \sum_{m\in \M_K(\llbracket s;t \llbracket)}\prod_{r\in m} a_{r}\cdot p(y^{r})
\end{align*}
where $A^k$ denotes the $k$-th power of matrix $A$ and $\left[A^{k}\right]_{st}$ its general term. Moreover, 
\begin{align*}
\mathcal{A}_K \defeq \{[A^k]_{1,t},[A^k]_{t,n+1} \}_{1\leqslant k \leqslant K \atop 2\leqslant t \leqslant N}
\end{align*}
 can be computed in $O(KN^2)$ time.
\end{proposition}
In particular, $[A^K]_{1,n+1} = C_K(a)\cdot p(y|K)$. Several quantities of interest share the same form: from $\mathcal{A}_K$, \cite{Rigaill2012} also derived exact formulas for the posterior probability of a change-point to occur at time $t$ or for the posterior probability that a given segment $r$ belongs to $m$ (see Section \ref{subsec:changepoint}). Classical Bayesian selection criteria for $K$ are also given. One can notice that $C_K(a)$ can be recovered by applying Proposition \ref{prop:rigaill} not to matrix $A$ but to a matrix defined similarly from $a$. For the uniform distribution on $\M_K$, \textit{i.e.} $a_r \equiv 1$, we get $C_K(a) = {{N-1}\choose{K-1}}$.\\

\cite{Fearnhead2006} worked under a slightly different model where $m$ is not chosen conditionally on $K$ but is instead drawn sequentially by specifying the probability mass function for the time between two successive change-points. They presented a filtering recursion to compute the marginal likelihood of the observations under their model where the integrations over parameters and segmentations are also uncoupled. \cite{Fearnhead2007} showed that on-line and exact inference is also tractable in this model.

\subsection{Tree-structured Graphical Models}
\label{subsec:tree}

In a multivariate setting, graphical models are used to describe complex dependence structures between the involved variables. A graphical model is given by a graph, either directed or not, and a family of distributions satisfying some Markov property with respect to this graph. We concentrate our attention on undirected graphical models, also called Markov random fields. We refer the reader to \citep{Lauritzen1996} for a complete overview on the subject. Let $V = \{1,...,p\}$ and $Y = (Y_1,...,Y_p)$ be a random vector taking values in a product space $\mathcal{X} = \bigotimes_{i = 1}^p \mathcal{X}_i$. We consider the set $\T$ of connected undirected graphs with no cycles. These graphs are called spanning trees. For $T\in \T$, we let $E_T$ denote the edges of $T$.
%
%

We consider a hierarchical model where one successively draws a tree $T$ in $\T$, the parameters $\theta$ of a distribution that factorises according to $T$, and finally a random vector $Y$ according to this distribution. The marginal likelihood of the observations under this model, where both $\theta$ and $T$ are integrated out, is given by
\begin{align*}
p(y) = \sum_{T\in \T}p(T)\int p(y|T,\theta)p(\theta|T)d\theta.
\end{align*} 
These integrations can be performed exactly and efficiently by choosing the right priors on $T$ and $\theta$ \citep{Meila06,Schwaller2015}. The distribution on trees is taken to be factorised on the edges,
\begin{align}
p(T) = \frac{1}{Z(b)}\prod_{\{i,j\} \in E_T} b_{ij}, \label{tree_prior}
\end{align}
where $b_{ij}$ are non-negative edge weights and 
\begin{align}
Z(b) \defeq \sum_{T\in \T} \prod_{\{i,j\} \in E_T}b_{ij} \label{eq:Z}
\end{align}
is a normalizing constant. The prior on $\theta$ has to be specified for all trees in $\T$. The idea is to require each of these priors to factorise on the edges and to specify a prior on $\theta_{ij}$ once and for all trees, $\theta_{ij}$ designating the parameters governing the marginal distribution of $(Y_i,Y_j)$. These priors must be chosen coherently, in the sense that, for all $i,j,k \in V$, the priors on $\theta_{ik}$ and $\theta_{jk}$ should induce the same prior on $\theta_k$. Some local Markov property is also needed. \cite{Schwaller2015} especially detailed three frameworks in which this can be achieved, namely multinomial distributions with Dirichlet priors, Gaussian distributions with normal-Wishart priors and copulas. \LS{}{We elaborate a little more on the particular case of Gaussian graphical models (GGMs). In a multivariate Gaussian setting, $\theta = (\mu,\Lambda)$ where $\mu$ and $\Lambda$ respectively stand for the mean vector and precision matrix of the distribution. A classical result on GGMs states that if the $(i,j)$-th term of the precision matrix is equal to zero, there is no edge between nodes $i$ and $j$. Thus, the support of $p(\theta |T)$ is the set of sparse positive definite matrices whose non-zero terms are given by the adjacency matrix of $T$. The distribution of $\theta|T$ can be defined for all trees at once by using a general normal-Wishart distribution defined on all positive-definite matrices \citep[Sec. 4.1.3]{Schwaller2015}. Marginal distributions of this normal-Wishart distributions are used to build distributions for $\{\theta|T\}_{T\in \T}$.}

When $p(\theta|T)$ is carefully chosen, the integration on $\theta$ can be performed independently from the integration on $T$ and $p(y|T)$ factorises on the edges of $T$:
\begin{align*}
p(y|T) = \prod_{i\in V}p(y_i)\prod_{\{i,j\} \in E_T} \frac{p(y_i,y_j)}{p(y_i)p(y_j)}
\end{align*}
where
\begin{align}
p(y_i,y_j) & = \int p(y_i,y_j | \theta_{ij})d\theta_{ij}, \label{eq:local_likelihood} \\
p(y_i)  &= \int p(y_i| \theta_{ij})d\theta_{i}. \nonumber
\end{align}
Computing $\{p(y|T)\}_{T\in \T}$ only requires $p(p+1)/2$ computations of low-dimensional integrals, \blue{where $p$ is the dimension of the model}. As both $p(T)$ and $p(y|T)$ factorise on the edges, integrating the likelihood over $T$ can be performed in $O(p^3)$ time. 
\begin{proposition}
\label{prop:seg_LK} The marginal likelihood is given by
\begin{align*}
p(y) = \frac{Z(\omega)}{Z(b)}\cdot\prod_{i\in V} p(y_i)
\end{align*}
where $Z(\cdot)$ is defined as in (\ref{eq:Z}) and $\omega$ is the posterior edge weight matrix whose general term is given by
\begin{align}
\omega_{ij} \defeq b_{ij}\frac{p(y_i,y_j)}{p(y_i)p(y_j)}. \label{eq:omega}
\end{align}
Moreover, $p(y)$ is obtained in $O(p^3)$ time from $b$ and $\omega$.
\end{proposition}
\begin{proof}
\begin{align*}
p(y) &= \sum_{T\in \T}p(y|T)p(T) \\
  & = \frac{1}{Z(b)}\left(\prod_{i\in V}p(y_i)\right)\sum_{T\in \T} \prod_{\{i,j\} \in E_T} b_{ij}\frac{p(y_i,y_j)}{p(y_i)p(y_j)} \\
  & = \frac{Z(\omega)}{Z(b)}\cdot\prod_{i\in V} p(y_i),
\end{align*}
with $\omega$ as defined above. As $Z(\cdot)$ can be computed in $O(p^3)$ time using the Matrix-Tree theorem, we get the announced complexity. \qed
\end{proof}
The posterior probability for an edge to belong to $T$, $P(\{i,j\} \in E_T | y)$, can also be obtained for all edges at once in $O(p^3)$ time \citep[Th. 3]{Schwaller2015}.

\section{Model \& Properties}
\label{sec:model}

Sections \ref{subsec:ppm} and \ref{subsec:tree} presented two models in which Bayesian inference requires us to integrate out a fundamentally discrete parameter (either the segmentation $m$ or the spanning tree $T$) and other (usually continuous) parameters $\theta$. In both situations, these integrations can be performed exactly and efficiently by uncoupling the problems. The integration over $\theta$ is performed ``locally" and the results are stored to be used in an algorithm that heavily relies on algebra to integrate over the discrete parameter. This is made possible by a careful choice of priors for both parameters. \blue{Our aim} is to show that these algebraic tricks can be combined to perform exact Bayesian inference of multiple change-points in the dependence structure of multivariate time-series.

\subsection{Model}
\label{subsec:model}

It is assumed that the observed data $y = \{y^{t}\}_{t=1}^N$ are a realisation of a multivariate random process $Y = \{Y^t\}_{t=1}^N$ of dimension $p \geqslant 2$. For $1 \leqslant t \leqslant N$, $Y^t = (Y^t_1,...,Y^t_p)$ is a multivariate random variable \blue{of dimension $p$} taking values in a product space $\XX = \bigotimes_{i=1}^p \XX_i$. We model $Y$ by a PPM where, \LS{}{at each time-point, observations $Y^t$ are modelled} by a tree-structured graphical model. If $m$ is a segmentation with $K$ segments, we let $\TT =\{T_k\}_{k=1}^K$ and $\theta =\{ \theta_r \}_{r\in m}$ respectively denote the tree structures and parameters for each segment. For $r\in m$, we also let $\kappa(r|m)$ denote the position of $r$ in $m$. 
Our model can then be written as follows:
\begin{align*}
p(m|K) & = \frac{1}{C_K(a)}\prod_{r\in m} a_r,\\
p(\TT | K) &= \prod_{k=1}^K p(T_k) = \frac{1}{Z(b)^K}\prod_{k=1}^K \prod_{\{i,j\} \in E_{T_k}}b_{ij},\\
p(\theta|m,\TT) & = \prod_{r\in m} p(\theta_r|T_{\kappa(r|m)}), & &\\
p(y|m,\theta,\TT ) & = \prod_{r\in m} \prod_{t\in r} p(y^t|T_{\kappa(r|m)},\theta_r).
\end{align*}
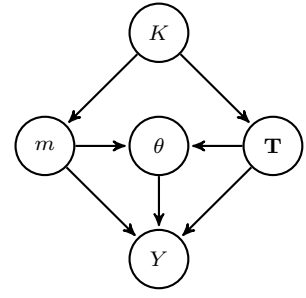
\begin{figure}[t]
\centering
\begin{tikzpicture}[->,>=stealth',shorten >=1pt,auto,node distance=1.5cm,
 		 thick,main node/.style={circle,fill=blue!0,draw,minimum size=22pt}];
 		 
 		 \tikzstyle{place}=[circle,draw=black!75,fill=black!10,minimum size=6mm];
 		 \node[main node] (1) { $\theta$}	; 
		\node[main node] (2)  [below of=1]{ $Y$};
		\node[main node] (3) [left of=1] { $m$};
		\node[main node] (4) [above of=1] { $K$};
  		\node[main node] (5) [right of=1] { $\TT$};
		\path[every node/.style={font=\sffamily\small}]
    		(4) edge node {} (3)
    		(4) edge node {} (5)
    		(3) edge node {} (1)
    		(5) edge node {} (1)
    		(1) edge node {} (2)
    		(3) edge node {} (2)
    		(5) edge node {} (2);
		\end{tikzpicture}
		\caption{Global graphical model. }
		\label{fig:global_model}
\end{figure}

For $r\in m$, $\{Y^t\}_{t\in r}$ are independent and identically distributed with structure $T_{\kappa(r|m)}$ and parameters $\theta_{r}$. The priors on $m$ and each of $T_k$ are respectively taken of the form given in  (\ref{eq:prior_segmentation}) and  (\ref{tree_prior}) through segment weights $a$ and edge weights $b$. 
{The distribution of $\theta_r|\{T_{\kappa(r|m)}=T\}$ is assumed to factorise over the edges of $T$, coherently between all \blue{spanning trees} $T\in \T$, as described in Section \ref{subsec:tree}.} A graphical representation of this model is given in Figure \ref{fig:global_model}.


\subsection{Factorisation Properties}

In the model that we have described, the marginal likelihood of the observation, conditional on $K$, is given by
\begin{align}
p(y|K) = \sum_{m\in \M_K} \sum_{\TT\in \T^{K}} \int p(y,m,\theta,\TT|K)d\theta.
\end{align}
Integrating out the discrete parameters $(m,\TT)$ requires to sum over a set of cardinality
\begin{align*}
|\M_K|\cdot |\T^K| = \binom{N-1}{K-1}\cdot p^{K(p-2)} \LS{}{\approx \left(\frac{Np^{p-2}}{K} \right)^{K}}.
\end{align*}
Nonetheless, the joint distribution of $(m,\theta,\TT)$, conditionally on $K$, factorises at different levels and integration can therefore be performed by combining the results given in Section \ref{sec:background}.

\LS{}{\begin{proposition}
The marginal likelihood $p(y|K)$ can be computed in $O(\max(K,p^3)N^2)$ time, \blue{where $p$ and $N$ respectively stand for the dimension of the model and the length of the series,} from the posterior edge weight matrices computed on all possible segments $r$, whose general terms are given by
\begin{align}
\omega^{(r)}_{ij} \defeq b_{ij}\frac{p(y^r_i,y^r_j)}{p(y^r_i)p(y^r_j)}. \label{eq:omega_r}
\end{align}
$p(y^r_i,y^r_j)$ and $p(y^r_i)$ are local integrals on $\theta$ computed on edges and vertices as defined in (\ref{eq:local_likelihood}).
\end{proposition}}

\begin{proof}

\blue{For any segmentation $m\in \M_K$ of $\llbracket 1;N \rrbracket$ into $K$ segments,} $\{(T_{\kappa(r|m)},\theta_r)\}_{r\in m}$ are independent, so that $p(y,m|K)$ can be written as
\begin{align*}
p(y,m|K) = \frac{1}{C_K(a)}\prod_{r\in m}a_r p(y^r),
\end{align*}
where $p(y^r)$ stands for the locally integrated likelihood of $y^r$ on segment $r$,
\begin{align}
p(y^r) = \sum_{T\in \T}p(T)\int \left( \prod_{t \in r} p(y^t|T,\theta) \right) p(\theta|T)d\theta.
\end{align}
Thus, $p(y,m|K)$ satisfies the factorability assumption required by \cite{Rigaill2012} and once the weighted segment likelihood matrix $A$, \LS{}{defined by
\begin{align*}
A_{s,t} = \left\lbrace\begin{array}{ll}
a_{\llbracket s;t\llbracket}\cdot p(y^{\llbracket s;t\llbracket}) & \textrm{if}~1\leqslant s<t \leqslant n+1, \\ 
0 & \textrm{otherwise},
\end{array} \right. 
\end{align*}}
is computed, Proposition \ref{prop:rigaill} can be used to gain access to $p(y|K)$ in $O(KN^2)$ time.

Computing matrix $A$ requires to integrate the likelihood over \blue{tree structure} $T$ and \blue{parameters} $\theta$ for all possible segments $r \subseteq\llbracket 1 ; N \rrbracket$. On each segment, we fall back to the tree-structured model described  in Section \ref{subsec:tree} and the integrated likelihood can be expressed using the local terms computed on vertices and edges that were defined in (\ref{eq:local_likelihood}). \LS{}{Indeed, for $r \subset \llbracket 1 ; N \rrbracket$, $p(y^r)$ is obtained through Proposition \ref{prop:seg_LK} applied to $\omega^{(r)}$ (defined in (\ref{eq:omega_r})):
\begin{align*}
p(y^r) = \frac{Z(\omega^{(r)})}{Z(b)}\cdot\prod_{i\in V} p(y^r_i).
\end{align*}
\blue{where we remind that $Z(\cdot)$ is the function giving the normalising constant of a tree distribution.} As a consequence, $A$ is computed in $O(p^3N^2)$ time from \blue{the posterior edge weight matrices} $\{ \omega^{(r)} \}_{r \subseteq \llbracket 1 ;N\rrbracket}$, hence the total complexity.} \qed
\end{proof}

The components of matrices $\omega^{(r)}$ result from the integration over $\theta$, which can be made separately and locally thanks to the assumptions made on its prior distribution in Section \ref{subsec:model}. This integration comes down to remove node $\theta$ in the global graphical model displayed in Figure \ref{fig:global_model}. \\
\LS{}{Marginal likelihood is only one of many quantities than  might be of concern in this model. Yet, once matrix $A$ has been 
calculated, other quantities of interest with respect to our model can be obtained at low cost. The next section provides a non-exhaustive list of such quantities.}

\section{Quantities of Interest} 
\label{sec:QoI}

\subsection{Change-point Location}
\label{subsec:changepoint}

For $m\in \M_K$, we let $1 = \tau_0 < \tau_1 < \dots < \tau_K = N$ denote the change-points of $m$ and, for $1 \leqslant k \leqslant K$, we let $r_k = \llbracket \tau_{k-1} ; \tau_k \llbracket$ denote its $k$-th segment. In this section we are interested in computing the posterior probabilities of the following subsets of $\M_K$,
\begin{align*}
\mathcal{B}_{K,k}(t) &\defeq \{m\in \M_K |  \tau_k =t \}, \\
\mathcal{B}_{K}(t) & \defeq \bigcup_{k=1}^K\mathcal{B}_{K,k}(t), \\
\mathcal{S}_{K,k}(\llbracket s ; t\llbracket) &\defeq \{m\in \M_K |  r_k = \llbracket s ; t \llbracket \} \\
\mathcal{S}_{K}(\llbracket s ; t \llbracket) & \defeq \bigcup_{k=1}^K\mathcal{S}_{K,k}(\llbracket s ; t \llbracket).
\end{align*}
Subsets $\mathcal{B}_{K}(t)$ and $\mathcal{S}_{K}(\llbracket s ; t \llbracket)$ are respectively the set of segmentations having a change-point at time $t$ and the set of segmentations containing segment $\llbracket s ; t \llbracket$. We let $B_{K,k}(t)$, $B_{K}(t)$, $S_{K,k}(\llbracket s ; t \llbracket)$ and  $S_{K}(\llbracket s ; t \llbracket)$ denote the respective posterior probabilities of these subsets.

\cite{Rigaill2012} showed that, with the convention that $\left[A^0\right]_{t1,t2} = 1$ for all $1\leqslant t_1<t_2 \leqslant N+1$, these probabilities could be expressed as
\begin{align*}
B_{K,k}(t) & = \frac{\left[A^{k}\right]_{1,t}\left[A^{K-k}\right]_{t,N+1}}{\left[A^{k}\right]_{1,N+1}}, \\
B_{K}(t) &= \sum_{k=1}^{K-1} B_{K,k}(t), \\
S_{K,k}(\llbracket s ; t \llbracket) & =  \frac{\left[A^{k-1}\right]_{1,s}A_{s,t}\left[A^{K-k}\right]_{t,N+1}}{\left[A^{k}\right]_{1,N+1}}, \\
S_{K}(\llbracket s ; t \llbracket) &= \sum_{k=1}^K S_{K,k}(\llbracket s ; t \llbracket).
\end{align*}
$\{B_{K,k}(t)\}_{t=1}^N$ provides the exact posterior distribution of the $k$-th change-point when $m$ has $K$ segments. Posterior segment probabilities $\{S_{K}(\llbracket s ; t \llbracket)\}_{1 \leqslant s < t \leqslant N+1}$ will be useful in the following.

Once $\{B_K(t)\}_{K\geqslant 2}$ is computed, the posterior probability  $B(t)$ of a change-point occurring at time $t$ integrated on $K$ is obtained as
\begin{align*}
B(t) = P(\cup_{K\geqslant 2}\mathcal{B}_K(t)|y) = \sum_{K\geqslant 2}p(K|y)B_K(t).
\end{align*}
The computation of the posterior distribution on $K$ is addressed below.

\subsection{Number of Segments}

The posterior distribution on $K$ can also be derived from Proposition \ref{prop:rigaill}.
\begin{proposition}
\begin{align*}
p(K|y) \propto \frac{p(K)[A^K]_{1,N+1}}{C_K(a)}.
\end{align*}
\end{proposition}
\begin{proof}
Bayes' rule states that $p(K|y) \propto p(K)p(y|K)$ and by Proposition \ref{prop:rigaill}, $p(y|K) = [A^K]_{1,N+1}/C_K(a)$. \qed
\end{proof}

The best segmentation \textit{a posteriori} can also be recovered efficiently by using matrix $A$ \blue{in the Segment Neighbourhood
Search algorithm} \citep{Auger1989}. Thus, if one's interest lies in retrieving the number of segments $K$, two estimators can be considered
\begin{align*}
\hat{K}_1 & = \argmax_{K}p(K|y), \\
\hat{K}_2 &= K(\argmax_{m}p(m|y)). 
\end{align*}
where $K(m)$ stands for the number of segments in $m$.
\subsection{Posterior Edge Probability}

\blue{For any segment $r \subseteq \llbracket 1;N\rrbracket$, it is possible to compute the posterior edge probabilities corresponding to segment $r$:
\begin{align*}
&P(\{i,j\} \in E_{T} | y^r), &\forall \{i,j\} \in \mathcal{P}_2(V),
\end{align*}
where $T$ is a random tree distributed as $T_1, \dots, T_K$.
Whenever $m$ is unknown, the segmentation can be integrated out to obtain instant posterior edge probabilities at any given time $t$. Conditionally on $K$, the instant posterior appearance probability of edge $\{i,j\}$ at time $t$ can be written as}
\begin{align*}
\p_{ij}^K(t) \defeq \sum_{m\in \M_K}p(m|y,K)P(\{i,j\} \in E_{T_{\kappa(t|m)}} | y,m),
\end{align*}
where $\kappa(t|m)$ gives the position of the segment containing $t$ in $m$.
\begin{proposition} The \blue{instant} posterior probability of edge $\{i,j\}$ \blue{at time t} is given by
\begin{align} 
\p_{ij}^K(t) = \sum_{r\ni t} S_K(r) P(\{i,j\} \in E_T | y^r). \label{eq:posterior_edge_matrix}
\end{align}
$\{ p_{ij}^K(t) \}_{1 \leqslant i,j \leqslant p \atop 1 \leqslant t \leqslant N}$ can be computed in $O\left(\max(K,p^3)N^2\right)$ time from $A$ and $\{\omega^{(r)}\}_{r\subseteq \llbracket 1;N\rrbracket}$.
\end{proposition}
\begin{proof}
This formula is similar to the one giving the posterior mean of the signal in \citep{Rigaill2012}. If $r\in m$ and $t\in r$, then $P(\{i,j\} \in E_{T_{\kappa(t|m)}} | y,m) = P(\{i,j\} \in E_T | y^r)$, hence the result. $\{S_K(r)\}_{r\in \llbracket 1 ; N \rrbracket}$ is obtained with complexity $O(KN^2)$ and $\{P(\{i,j\} \in E_T | y^r)\} _{r\in \llbracket 1 ; N \rrbracket}$ with complexity $O(p^3N^2)$, and that gives an upper bound on total complexity. \qed
\end{proof}

One could be interested in computing the posterior probability for an edge to keep the same status throughout time when $m$ is integrated out, given $K$. Nonetheless, it would require to integrate on subsets of $\M_K\otimes \T^K$ that are in direct contradiction with the factorability assumption, making the results that we have presented so far useless. \LS{}{Indeed, we would effectively be introducing dependency between segments, thus breaking up the factorability of $p(y,m)$ with respect to $r\in m$. In this situation, Proposition \ref{prop:rigaill} can no longer be used. A drastic workaround is to work under a fixed segmentation instead of integrating out $m$, and this is what we do in the following section.}

\section{Edge Status \& Structure Comparisons}
\label{sec:edge_status}

{We now turn to the specific case where }$m$ is known and has $K$ segments $(r_1,\dots,r_K)$. {This situation is far less general than the framework we considered until now. Still, it corresponds to some practical situations where segment comparison is interesting and for which further exact inference can be carried out.}

\subsection{Edge Status Comparison} 
Let $i,j$ be two distinct nodes in $V$.

 We are interested in computing the posterior probability of the subsets of $\T^K$ defined by
\begin{align*}
\E_{ij}^+ & = \{\blue{\TT = (T_1, \dots , T_K)}| \forall k \in \llbracket 1;K \rrbracket, \{i,j\} \in E_{T_k} \}, \\
\E_{ij}^- & = \{ \blue{\TT = (T_1, \dots , T_K)}| \forall k \in \llbracket 1;K \rrbracket, \{i,j\} \notin E_{T_k} \}, \\
\E_{ij} &= \E_{ij}^+ \cup \E_{ij}^- ,
\end{align*}
that respectively correspond to the situations where edge $\{i,j\}$ is always present, always absent, or has the same status in all trees. If $\TT$ belongs to $\overline{\E_{ij}} = \T^K \setminus \E_{ij}$, it means that there exists two segments in which $\{i,j\}$ does not have the same status. We let $(q_0^-,\overline{q_0},q_0^+)$  respectively denote the prior probabilities of $\E_{ij}^-$, $\overline{\E_{ij}}$ and $\E_{ij}^+$. These probabilities can be written as
\begin{align*}
q_0^- & = \prod_{k=1}^K P(\{i,j\} \notin E_{T_k}) = P(\{i,j\} \notin E_{T})^K, \\
 q_0^+ & = P(\{i,j\} \in E_{T})^ K, \hspace{0.5cm} \overline{q_0}  = 1 -  q_0^- - q_0^+,
\end{align*}
where $T$ is a tree distributed as $T_1,\dots,T_K$,
and are obtained for all edges at once in $O(p^3)$ time by computing the prior edge probability matrix $(P(\{i,j\} \notin E_{T}))_{1\leqslant i \leqslant j \leqslant p}$.

 Posterior probabilities $(q^-,\overline{q},q^+)$ for $\E_{ij}^-$, $\overline{\E_{ij}}$ and $\E_{ij}^+$ can be computed similarly but one posterior edge probability matrix has to be calculated per segment:
 \begin{align*}
q^- & = \prod_{k=1}^K P(\{i,j\} \notin E_{T_k}|y^{r_k}) , \\
 q^+ & = \prod_{k=1}^K P(\{i,j\} \in E_{T_k}|y^{r_k}), \hspace{0.5cm} \overline{q}  = 1 -  q^- - q^+,
\end{align*}

However, if the prior distribution on trees is not strongly peaked, as events $\E_{ij}^+$ and $\E_{ij}^-$ only account for a relatively small number of tree series in $\T^K$, $q_0^-$ and $q_0^+$ (as well as $q^-$ and $q^+$) will always be very small. To allow some control on the prior probabilities of these events, we use a random variable $\epsilon_{ij}$ taking values $\{-1;0;1\}$ with probabilities $(\lambda^-,\overline{\lambda},\lambda^+)$ and  explicitly controlling the status of edge $\{i,j\}$ in all trees:
\begin{align*}
p(\TT|\epsilon_{ij}) = \left\lbrace \begin{array}{ll}
p(\TT|\E_{ij}^+) & \textrm{if}~\epsilon_{ij} = 1, \\ 
p(\TT|\overline{\E_{ij}}) & \textrm{if}~\epsilon_{ij} = 0, \\ 
p(\TT|\E_{ij}^-) & \textrm{if}~\epsilon_{ij} = -1.
\end{array} \right.
\end{align*}
We obtain the model described in Figure \ref{fig:model_edge_comp}, in which
\begin{align*}
p(y) = \lambda^+p(y|\E_{ij}^+) + \lambda^-p(y|\E_{ij}^-) + \overline{\lambda}p(y|\overline{\E_{ij}}).
\end{align*}

\begin{figure}[t]
\centering
		\begin{tikzpicture}[->,>=stealth',shorten >=1pt,auto,node distance=1cm,
 		 thick,main node/.style={circle,fill=none,draw,font=\small,minimum size=18pt,inner sep=0pt}]
 		 
 		 \tikzstyle{place}=[circle,draw=none,fill=white,minimum size=5pt,font=\tiny]; 

 		\node[main node,fill=black!10] (1) {\large $\epsilon_{ij}$};
		\node[main node] (2) [below of=1] {$T_k$};
		\node[main node] (3) [below of=2] {$\theta_k$};
		\node[main node] (4) [below of=3] { $Y^{r_k}$};
		
		\node[place] (5) [right of=2] {$...$};
  		\node[place] (6) [right of=3] {$...$};
  		\node[place] (7) [right of=4] {$...$};
  		
  		\node[main node] (8) [right of=5] {$T_K$};
  		\node[main node] (9) [right of=6] {$\theta_K$};
  		\node[main node] (10) [right of=7] { $Y^{r_K}$};
  		
  		\node[place] (l5) [left of=2] {$...$};
  		\node[place] (l6) [left of=3] {$...$};
  		\node[place] (l7) [left of=4] {$...$};
  		
  		\node[main node] (l8) [left of=l5] {$T_1$};
  		\node[main node] (l9) [left of=l6] {$\theta_1$};
  		\node[main node] (l10) [left of=l7] { $Y^{r_1}$};

		\path[every node/.style={font=\sffamily\small}]
    		(1) edge node {} (2)
        		edge node {} (8)
        		edge node {} (l8)
    		(2) edge node {} (3)
    		(3) edge node {} (4)
    		(8) edge node {} (9)
    		(9) edge node {} (10)
    		(l8) edge node {} (l9)
    		(l9) edge node {} (l10);
		\end{tikzpicture}
\caption{Model for edge status comparison.}
\label{fig:model_edge_comp}
\end{figure}
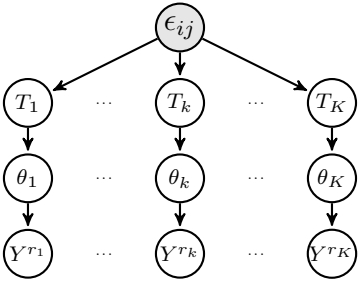

\begin{proposition}
\label{prop:edge_status}
The vector of posterior probabilities for $\epsilon_{ij}$ is proportional to $\left(\lambda^-\frac{q^-}{q_0^-},\overline{\lambda}\frac{\overline{q}}{\overline{q_0}},\lambda^+\frac{q^+}{q_0^+} \right)$.
\end{proposition}
\begin{proof}We have that
\begin{align*} 
p(\epsilon_{ij} = 1 | y)  &= \lambda^+\frac{p(y|\E_{ij}^+)}{p(y)} \\
 & = \frac{\lambda^+p(y|\E_{ij}^+)}{\lambda^+p(y|\E_{ij}^+) + \lambda^-p(y|\E_{ij}^-) + \overline{\lambda}p(y|\overline{\E_{ij}})} \\
  & = \frac{\lambda^+\frac{q^+}{q_0^+}}{\lambda^+\frac{q^+}{q_0^+} + \lambda^-\frac{q^-}{q_0^-} + \overline{\lambda}\frac{\overline{q}}{\overline{q_0}}}.
\end{align*}
We reason similarly with $p(\epsilon_{ij} = -1 | y)$ to get the result. \qed
\end{proof}

\subsection{Structure Comparison}

The same reasoning can be applied for the global event
\begin{align*}
\E & = \{ \blue{\TT = (T_1, \dots , T_K)} | \exists T \in \T, \forall k \in \llbracket 1;K \rrbracket, T_k = T \}{,}
\end{align*}
{which corresponds to a constant dependency structure across all segments (we remind that, in this section, the segments are known \textit{a priori}), with possible changes for the parameters.}
The prior probability of $\E$ is given by
\begin{align*}
q_0 \defeq P(\E) = \frac{1}{Z(b)^K}\sum_{T\in \T} \prod_{\{i,j\} \in E_T} b_{ij}^K = \frac{Z(b^{\odot K})}{Z(b)^K},
\end{align*}
where $b^{\odot K}$ stands for the element-wise $K$-th power of matrix $b$. On each segment $r_k$, the posterior distribution on trees factorises as
\begin{align*}
p(T_k|y^{r_k}) = \frac{1}{Z(\omega^{(k)})}\prod_{\{i,j\} \in T_k} \omega_{ij}^{(k)},
\end{align*}
and the posterior probability of $\E$ is therefore given by
\begin{align*}
q \defeq \sum_{T\in \T}\prod_{k=1}^K p(T|y^{r_k}) = \frac{Z\left(\bigodot_k \omega^{(k)}\right)}{\prod_k Z(\omega^{(k)})}
\end{align*}
where $\bigodot$ denotes the element-wise matrix product.

Just as in the edge status comparison, we let a binary variable $\epsilon \sim \mathcal{B}(\pi)$ control the prior probability of $\E$, with $p(\TT|\epsilon = 1) = p(\TT | \E)$, and derive a similar formula for the posterior distribution of $\epsilon$.

\begin{proposition} $\epsilon | y \sim \mathcal{B}(\pi^*)$
with $\pi^* \defeq \frac{\pi \frac{q}{q_0}}{\pi \frac{q}{q_0} + (1-\pi)\frac{1-q}{1-q_0}}$.
\end{proposition}
\begin{proof}
Similar to Proposition \ref{prop:edge_status}. \qed
\end{proof}

\section{Simulations}
\label{sec:simulation}


Our approach was especially concerned with explicitly modelling the structure of the graphical model within each segment, but a simpler model could be considered in which the structure remains implicit. In a Gaussian setting, that would mean that the precision matrix governing the distribution on a given segment would be drawn without any zero-term constraints. One goal of this simulation study is to show how both models (with and without structure constraints) comparatively behave when one is interested in retrieving the number of segments or the location of the change-points.

Another concern addressed by these simulations is the cost of the tree assumption when the true model is not tree-structured. How well can the number of segments, the change-points or even the structures be recovered when the true networks are not trees?

\subsection{Simulation Scheme}

\begin{figure*}[b]
\vspace{0.1cm}
\begin{minipage}{0.04\textwidth}
~
\end{minipage}
\begin{minipage}{0.32\textwidth}
\centering \hspace{0.3cm}Tree
\end{minipage}
\begin{minipage}{0.32\textwidth}
\centering \hspace{0.3cm}Erd\"{o}s-R\'{e}nyi, $p_C = 2/p$
\end{minipage}
\begin{minipage}{0.32\textwidth}
\centering \hspace{0.3cm}Erd\"{o}s-R\'{e}nyi, $p_C = 4/p$
\end{minipage}

\includegraphics[width=0.04\textwidth]{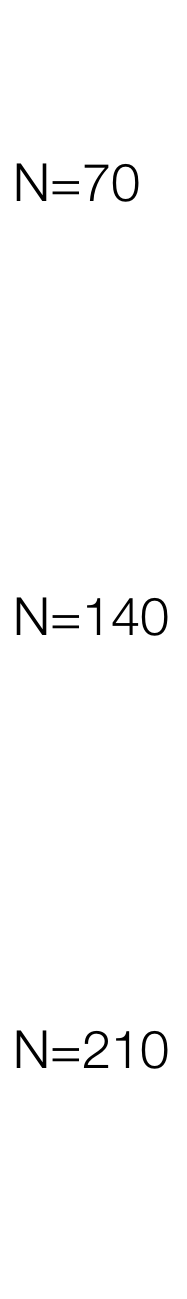}
\includegraphics[width=0.32\textwidth]{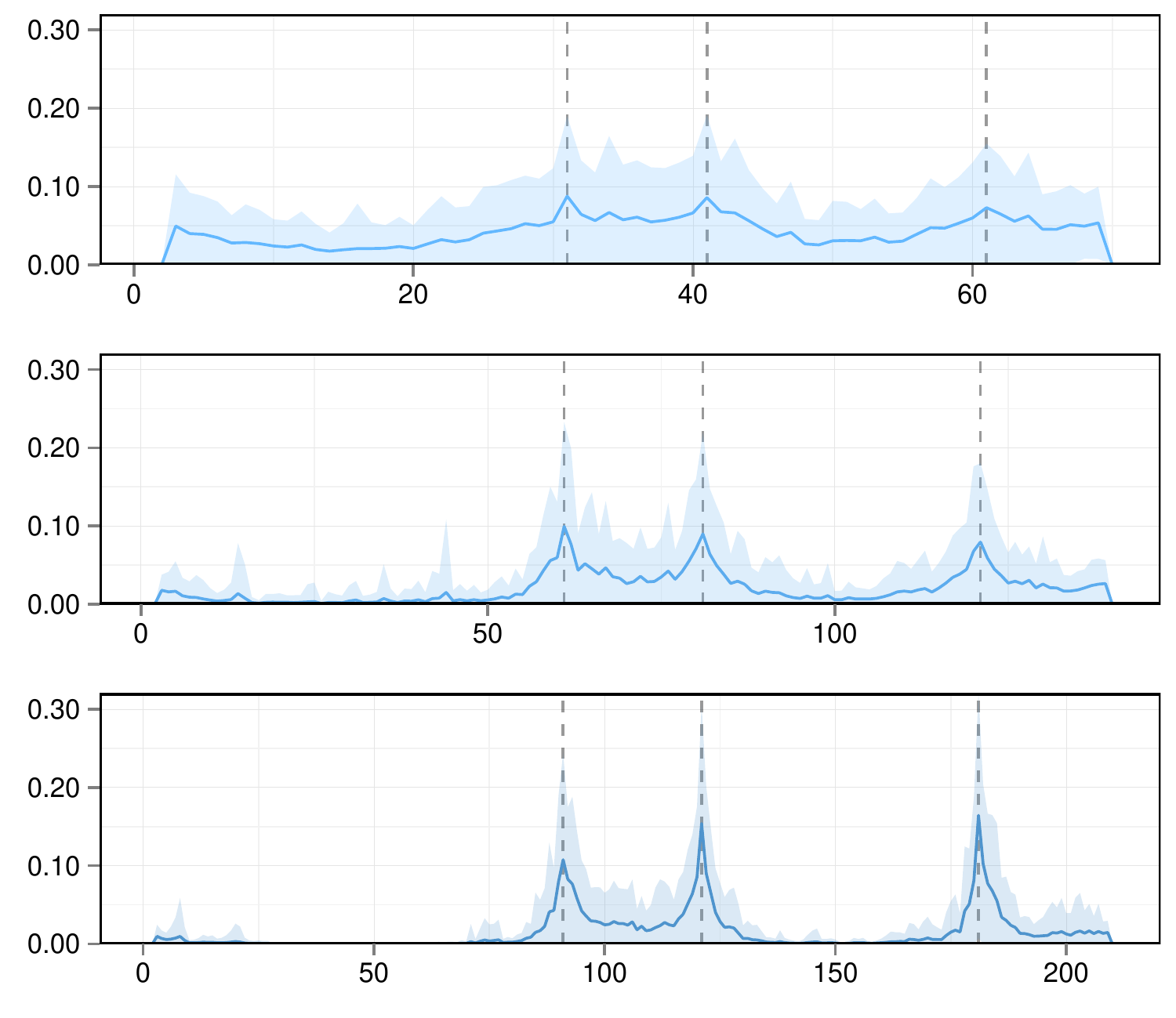}
\includegraphics[width=0.32\textwidth]{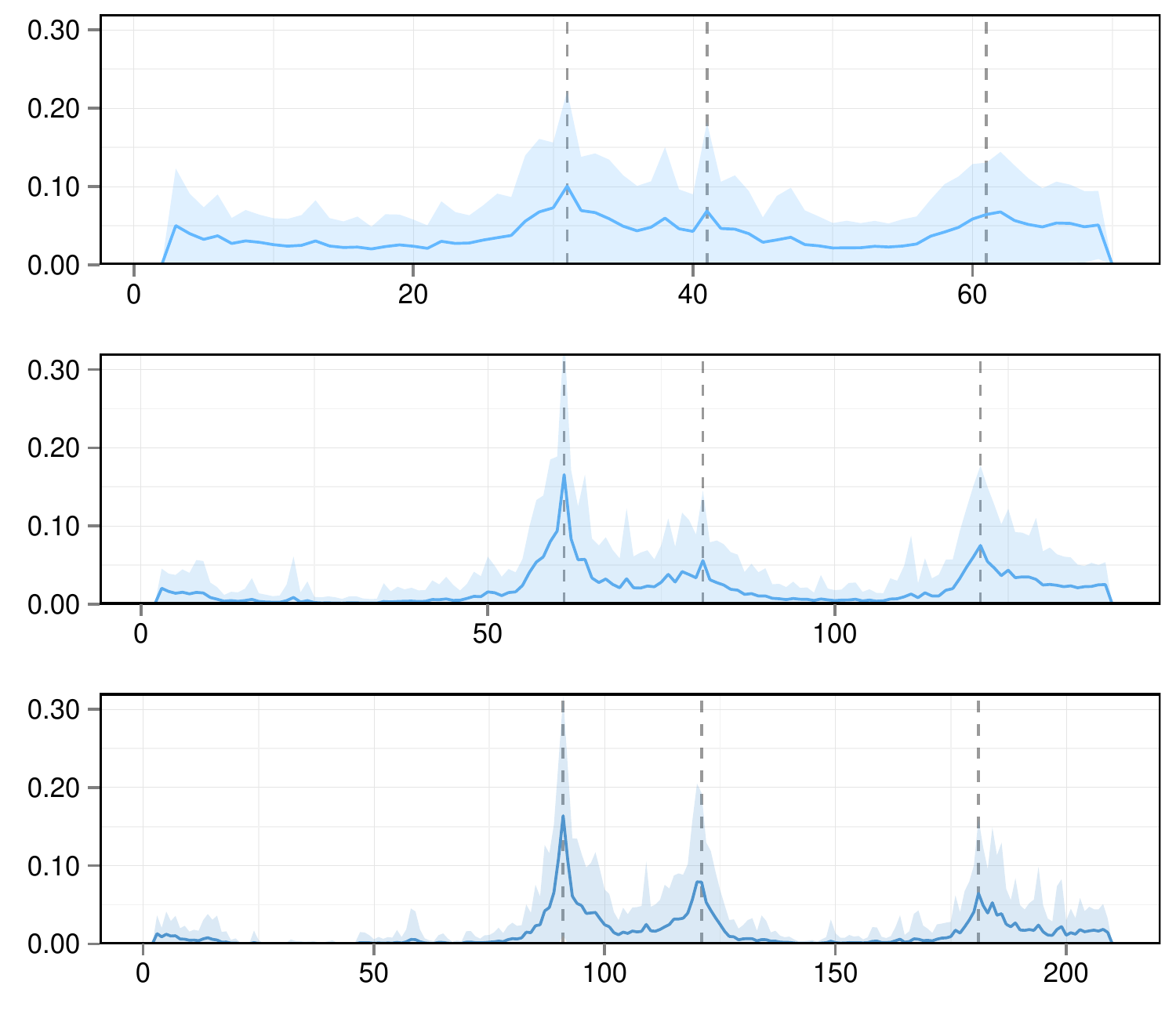}
\includegraphics[width=0.32\textwidth]{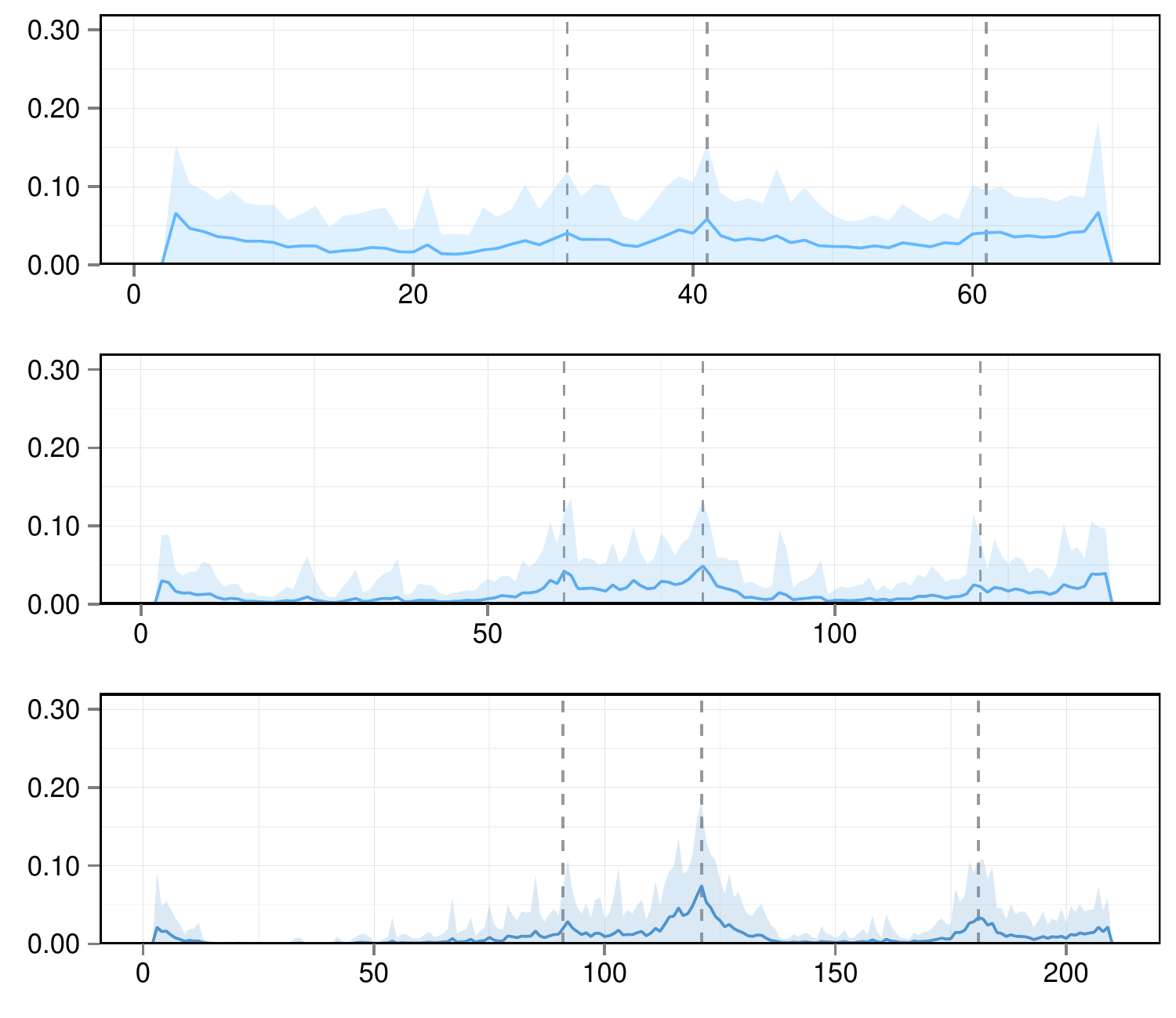}

\vspace{-0.1cm}

\includegraphics[width=0.04\textwidth]{sample_size.pdf}
\includegraphics[width=0.32\textwidth]{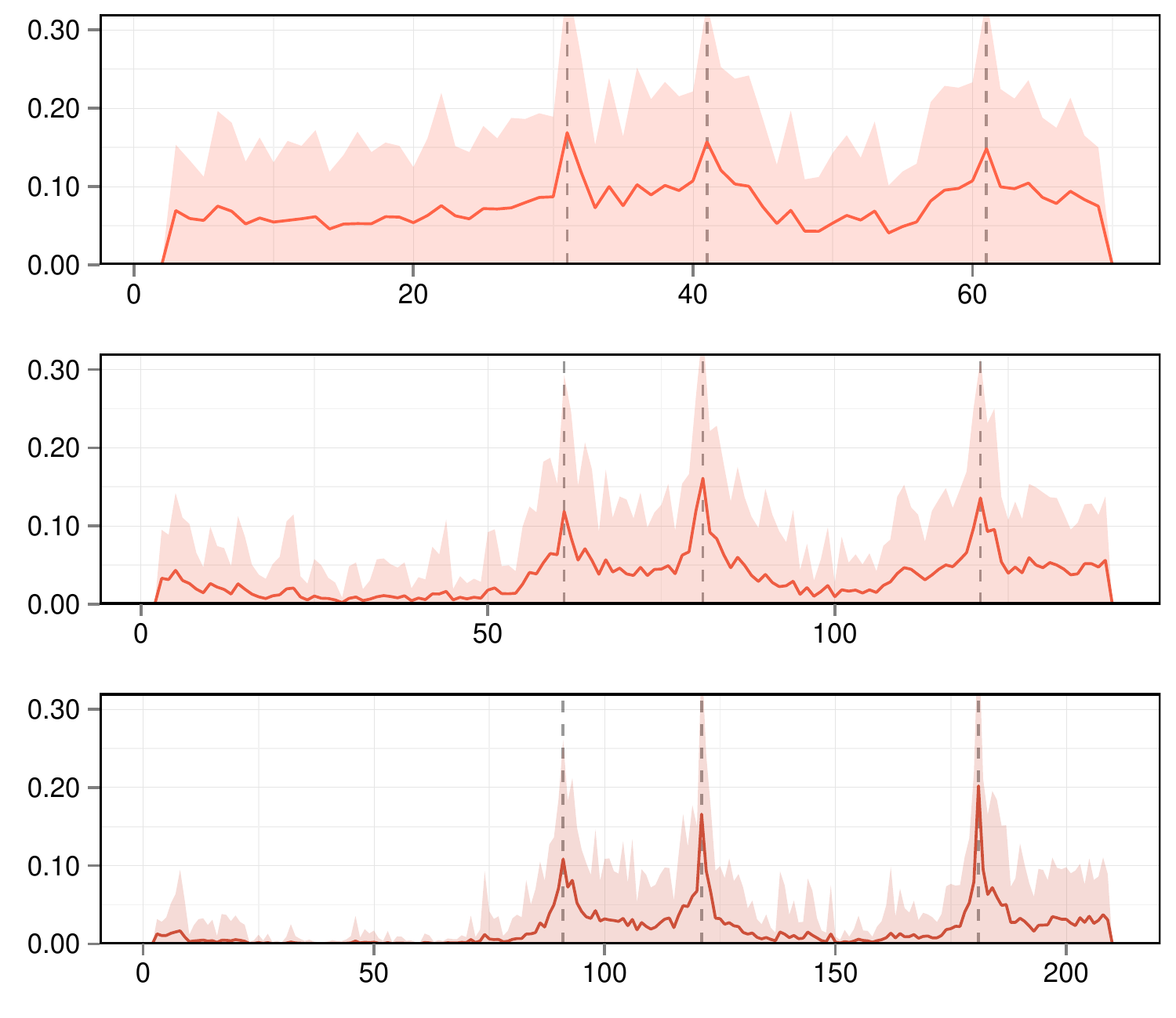}
\includegraphics[width=0.32\textwidth]{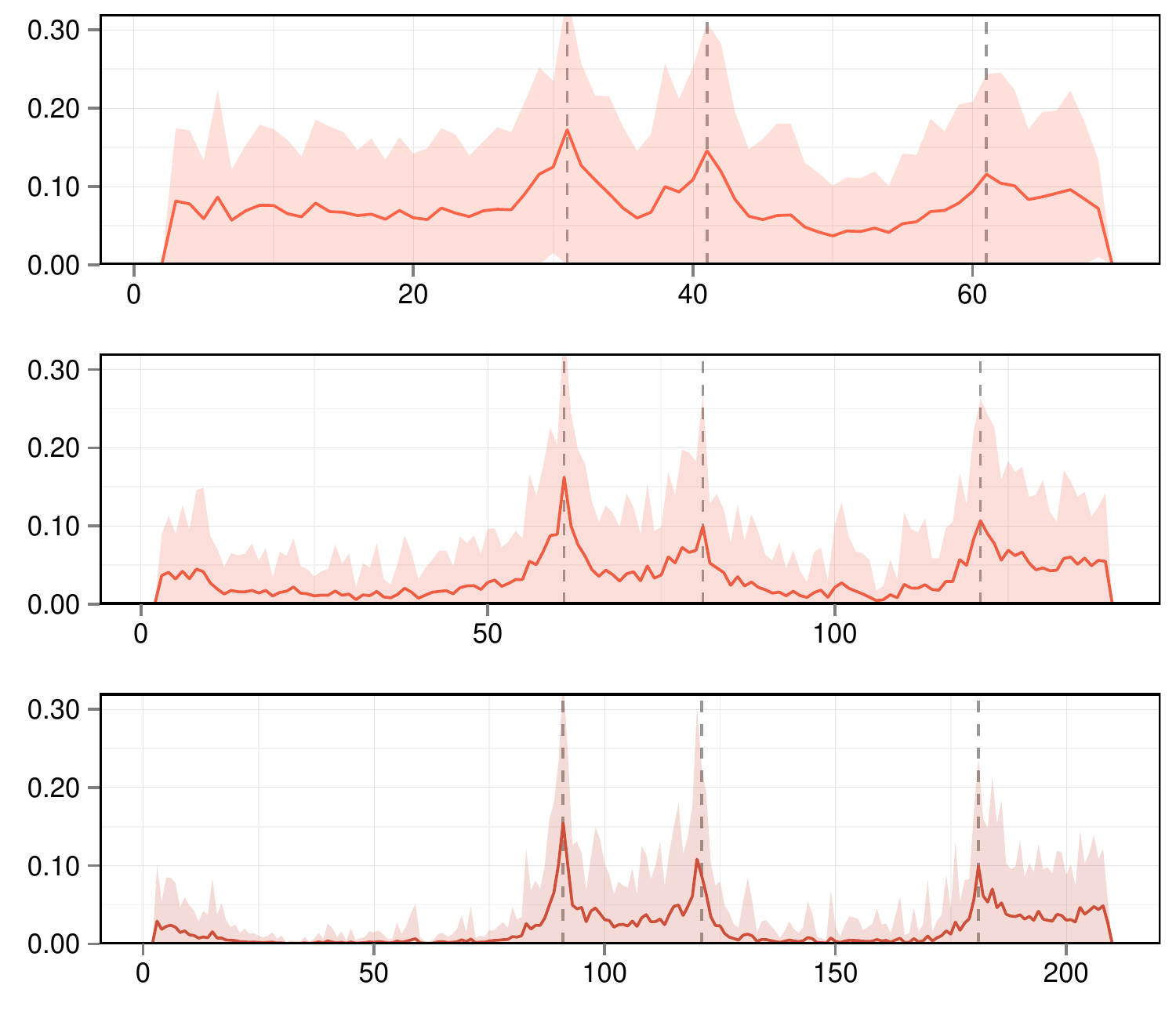}
\includegraphics[width=0.32\textwidth]{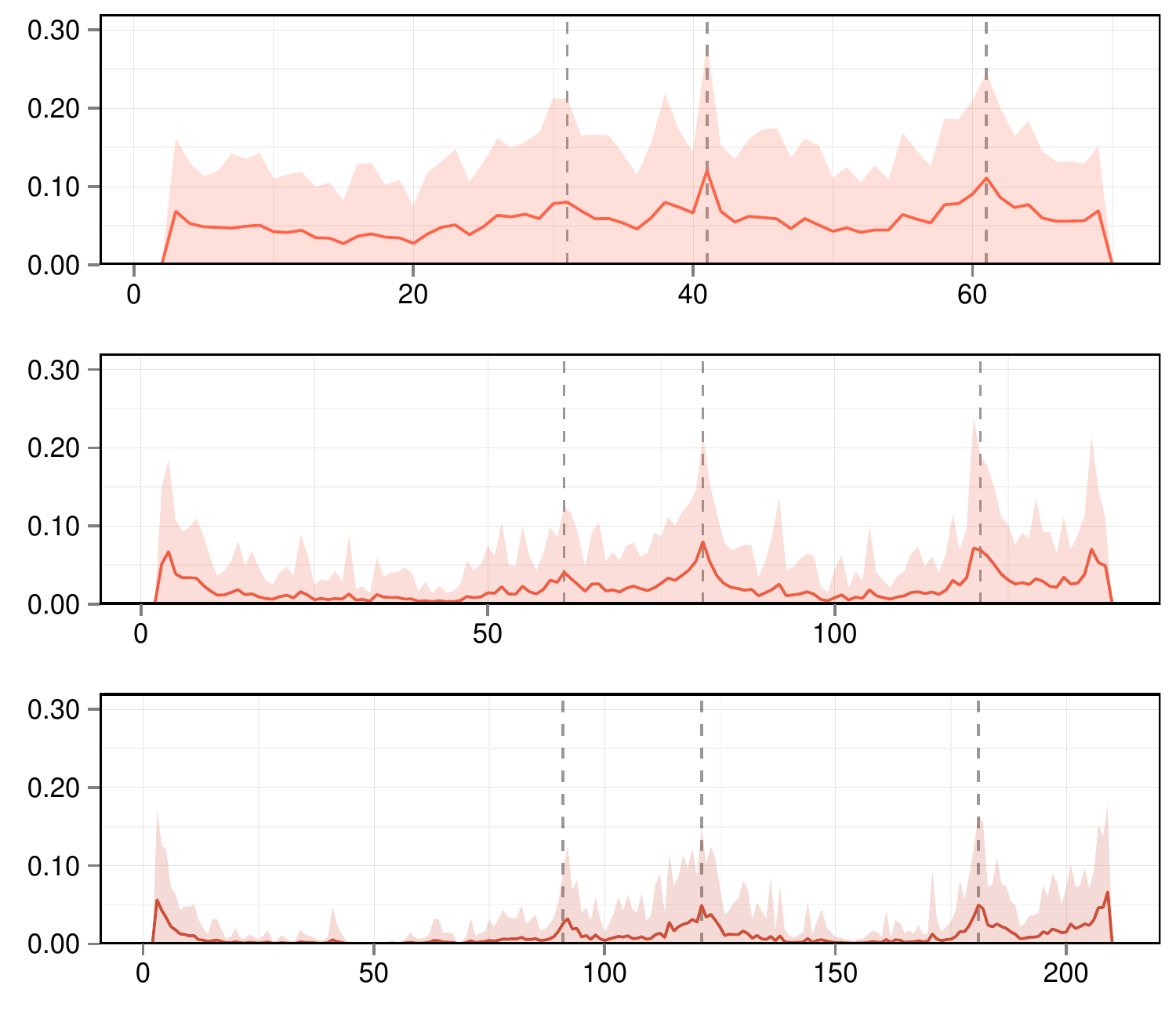}

\vspace{-0.2cm}

\begin{minipage}{0.04\textwidth}
~
\end{minipage}
\begin{minipage}{0.32\textwidth}
\centering \hspace{0.3cm}Tree
\end{minipage}
\begin{minipage}{0.32\textwidth}
\centering \hspace{0.3cm}Erd\"{o}s-R\'{e}nyi, $p_C = 2/p$
\end{minipage}
\begin{minipage}{0.32\textwidth}
\centering \hspace{0.3cm}Erd\"{o}s-R\'{e}nyi, $p_C = 4/p$
\end{minipage}

\caption{Posterior probability of observing a change-point for the tree-structured model (blue) and for the full model (red). The curve represents the mean value obtained from the 100 samples and the ribbon gives the standard deviation. 
}
\label{fig:simu_chg_pt}
\end{figure*}

\begin{figure*}[p]
\vspace{2cm}

\includegraphics[width=0.04\textwidth]{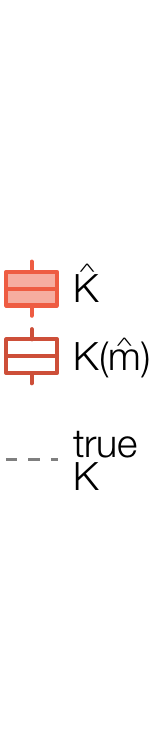}
\includegraphics[width=0.32\textwidth]{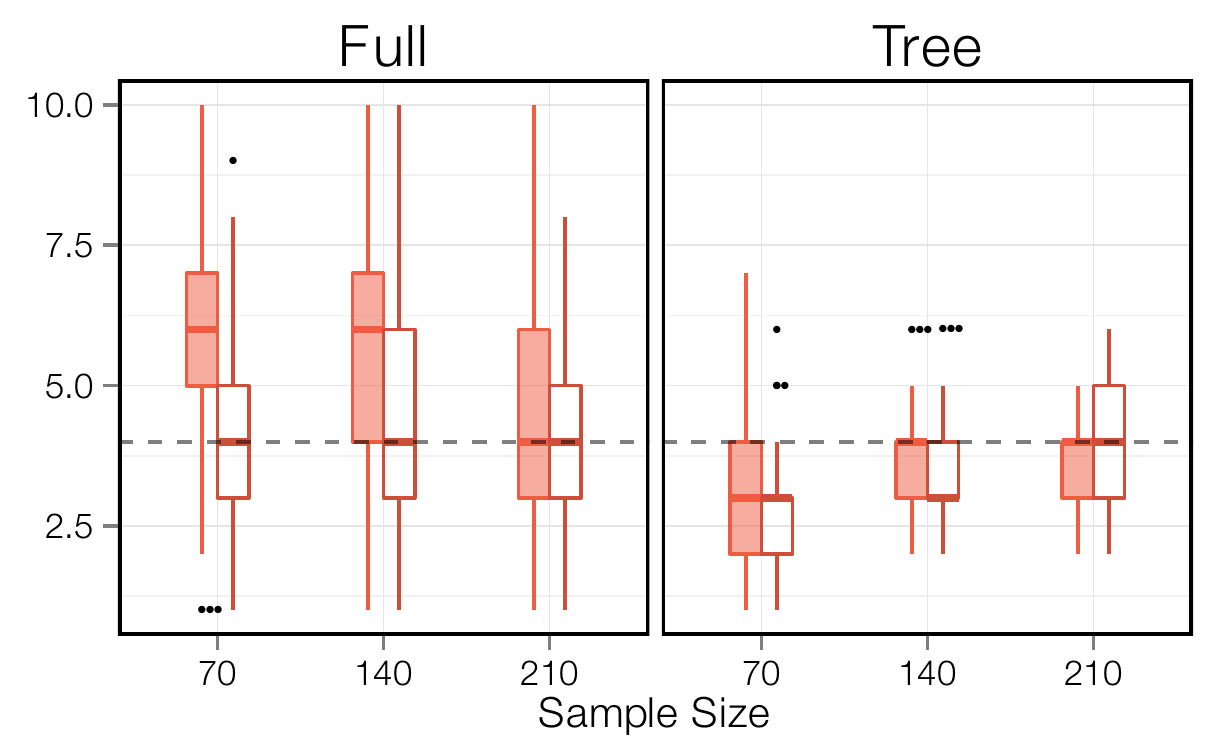}
\includegraphics[width=0.32\textwidth]{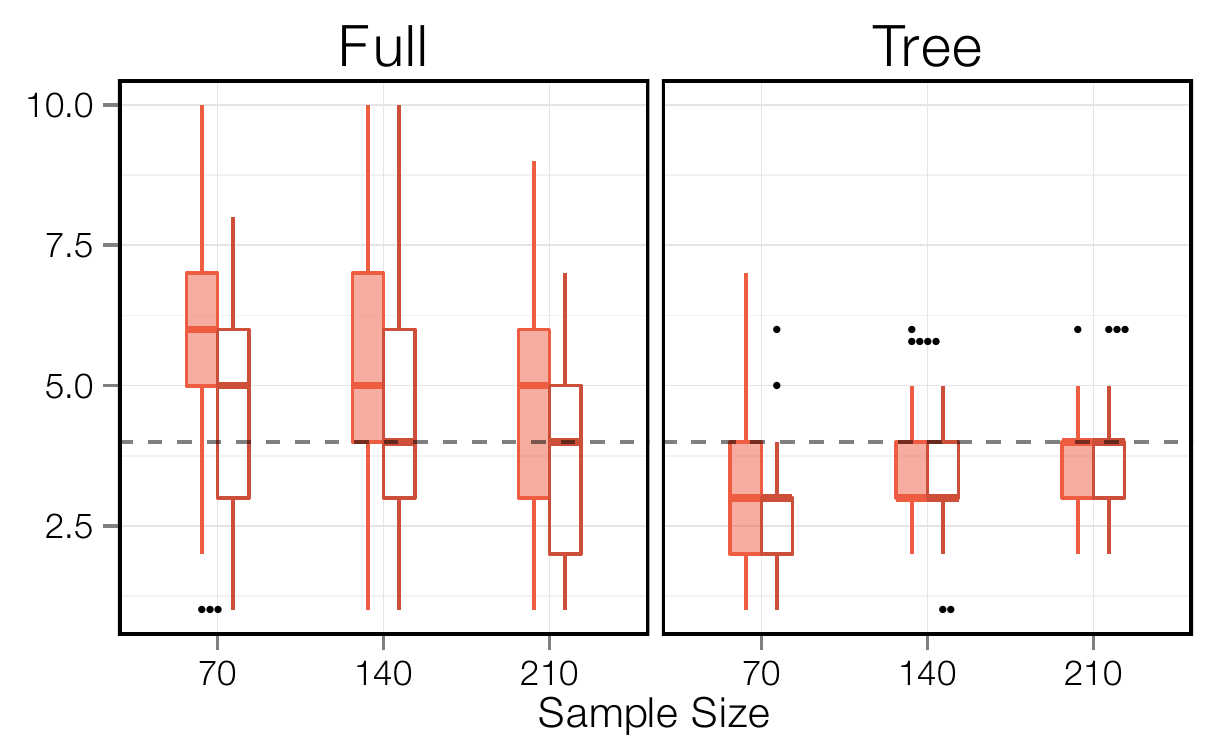}
\includegraphics[width=0.32\textwidth]{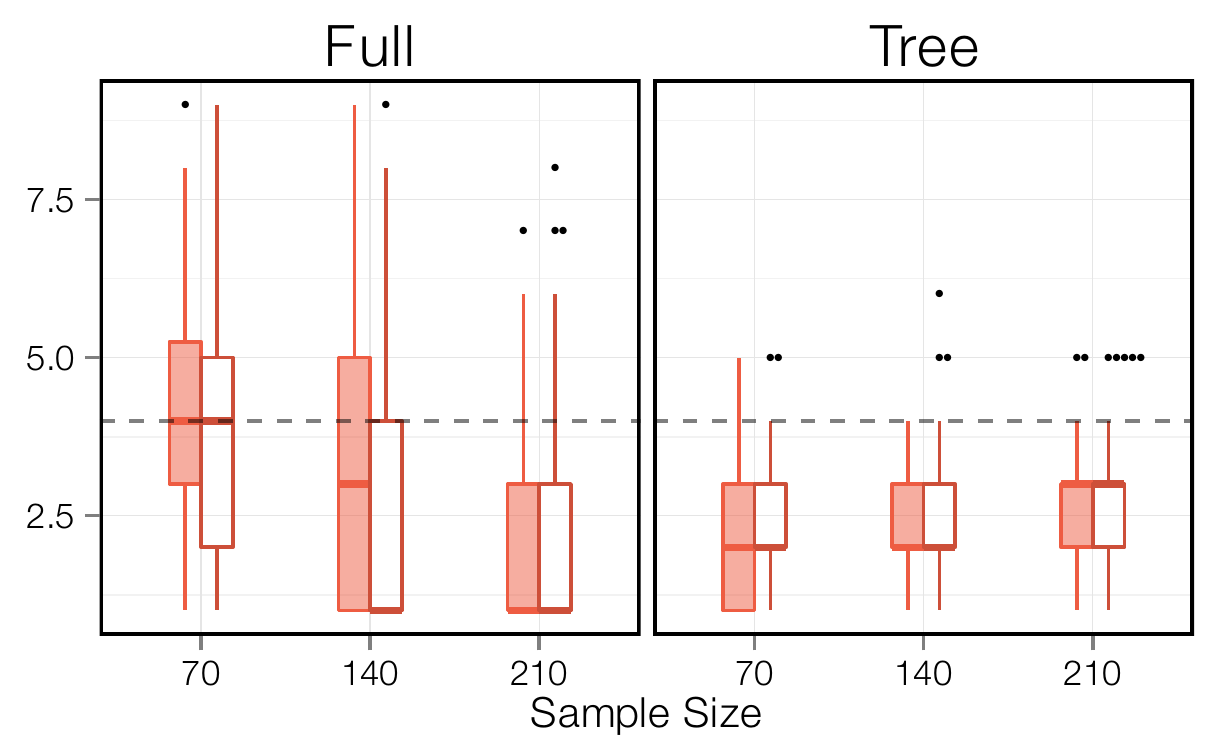}

\begin{minipage}{0.04\textwidth}
~
\end{minipage}
\begin{minipage}{0.32\textwidth}
\centering Tree
\end{minipage}
\begin{minipage}{0.32\textwidth}
\centering Erd\"{o}s-R\'{e}nyi, $p_C = 2/p$
\end{minipage}
\begin{minipage}{0.32\textwidth}
\centering Erd\"{o}s-R\'{e}nyi, $p_C = 4/p$
\end{minipage}

\vspace{-0.2cm}

\caption{Boxplot of $\hat{K} = \argmax_K p(K|y)$ and $K(\hat{m}) = K(\argmax_m p(m|y))$ against sample size $N$ for the full model (Full) and the tree-structured model (Tree).}
\label{fig:simu_K}
\end{figure*}

\begin{figure*}[p]

\includegraphics[width=0.04\textwidth]{sample_size.pdf}
\includegraphics[width=0.32\textwidth]{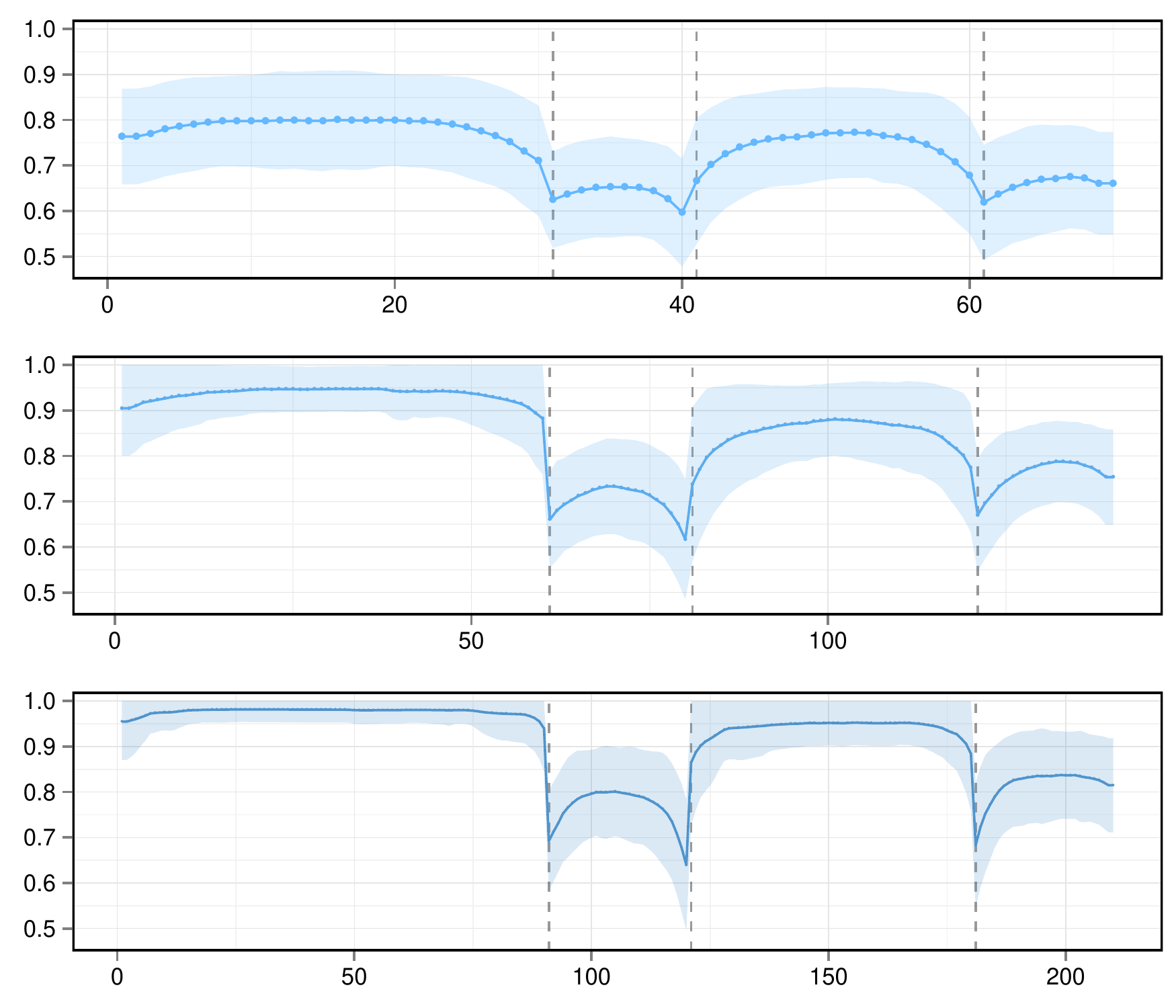}
\includegraphics[width=0.32\textwidth]{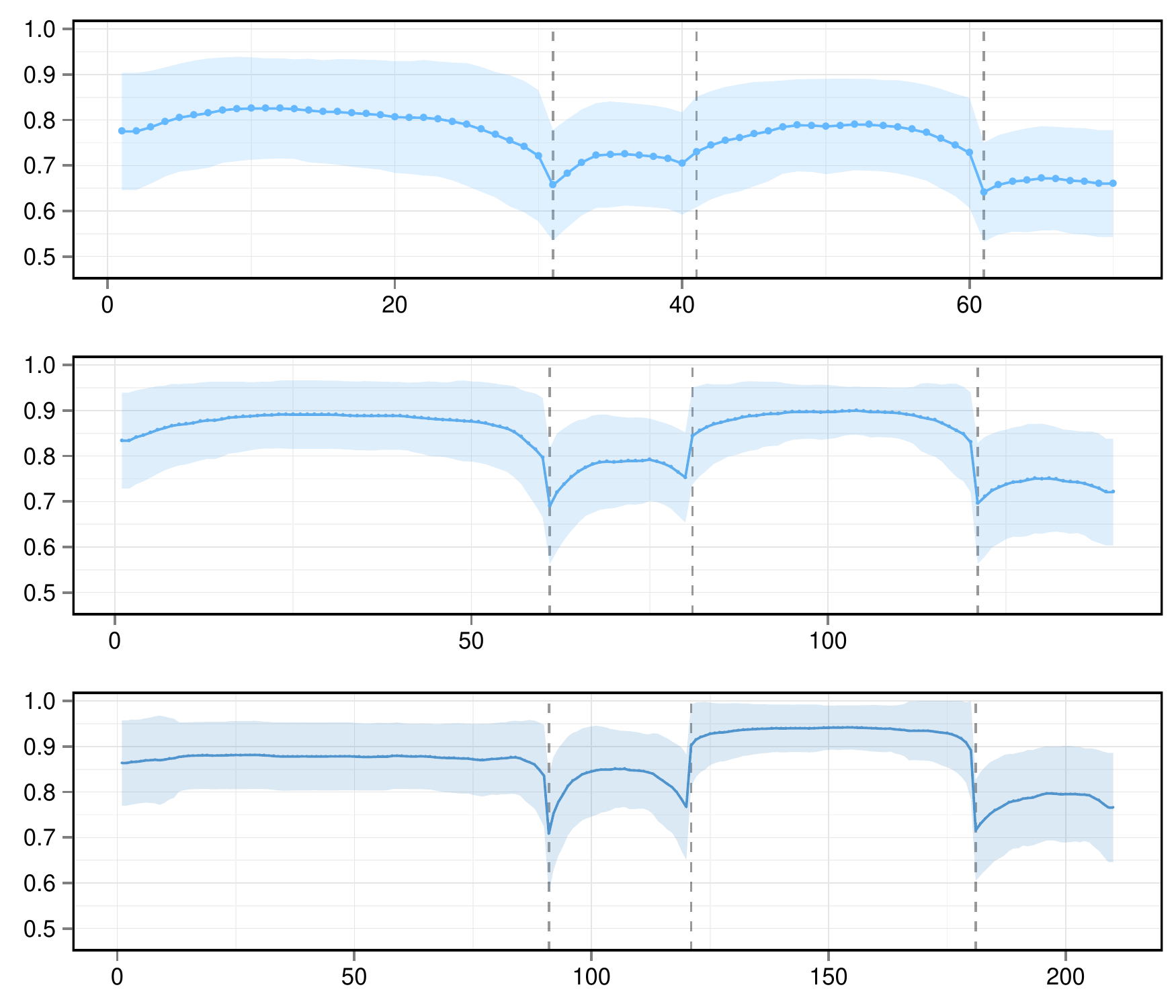}
\includegraphics[width=0.32\textwidth]{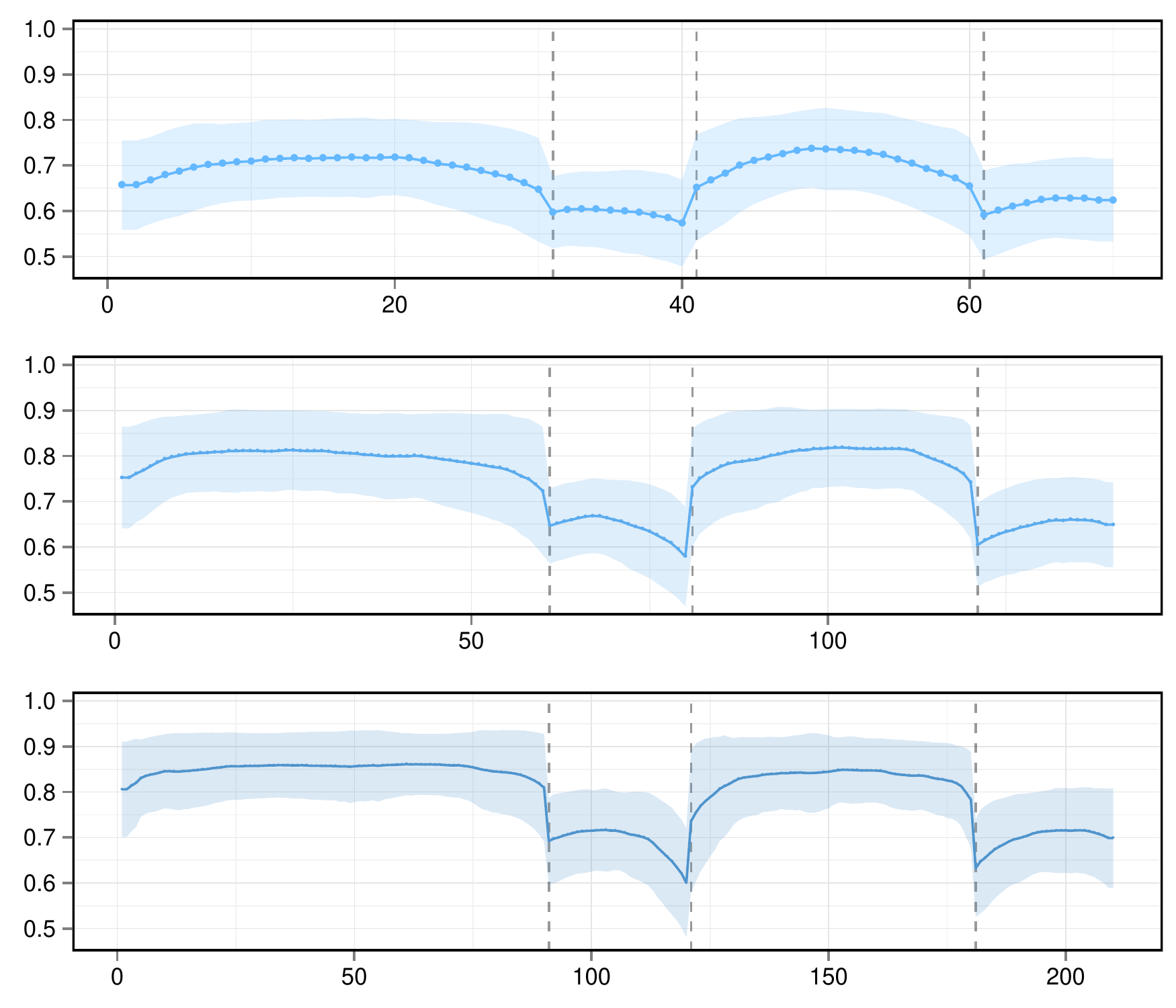}

\begin{minipage}{0.04\textwidth}
~
\end{minipage}
\begin{minipage}{0.32\textwidth}
\centering \hspace{0.3cm}Tree
\end{minipage}
\begin{minipage}{0.32\textwidth}
\centering \hspace{0.3cm}Erd\"{o}s-R\'{e}nyi, $p_C = 2/p$
\end{minipage}
\begin{minipage}{0.32\textwidth}
\centering \hspace{0.3cm}Erd\"{o}s-R\'{e}nyi, $p_C = 4/p$
\end{minipage}

\caption{Area under the ROC curve computed for the posterior edge probability matrix $\left[\p_{ij}^K(t)\right]_{i,j=1}^p$ with respect to the true adjacency matrix at time $t$. We set $K$ to the true number of segments ($K=4$). The curve represents the mean value obtained from the 100 samples and the ribbon gives the standard deviation. }

\label{fig:simu_aucroc}
\end{figure*}
\begin{figure*}[p]
\begin{minipage}{0.1\textwidth}
~
\end{minipage}
\begin{minipage}{0.298\textwidth}
\centering \hspace{0.6cm}Tree
\end{minipage}
\begin{minipage}{0.298\textwidth}
\centering \hspace{0.6cm}Erd\"{o}s-R\'{e}nyi, $p_C = 2/p$
\end{minipage}
\begin{minipage}{0.298\textwidth}
\centering \hspace{0.6cm}Erd\"{o}s-R\'{e}nyi, $p_C = 4/p$
\end{minipage}


\includegraphics[width=0.1\textwidth]{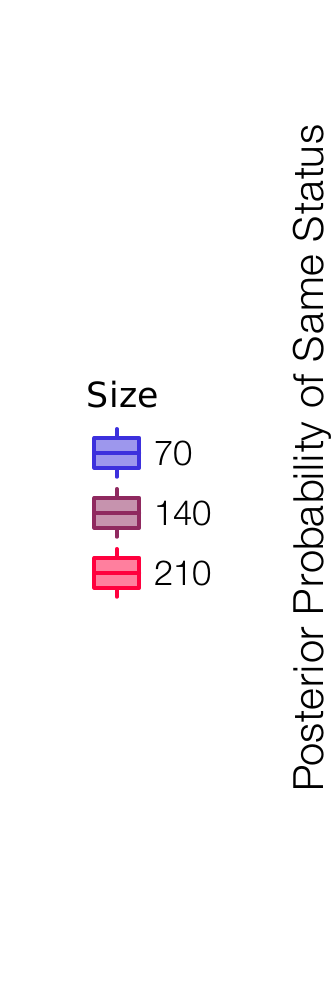}
\includegraphics[width=0.298\textwidth]{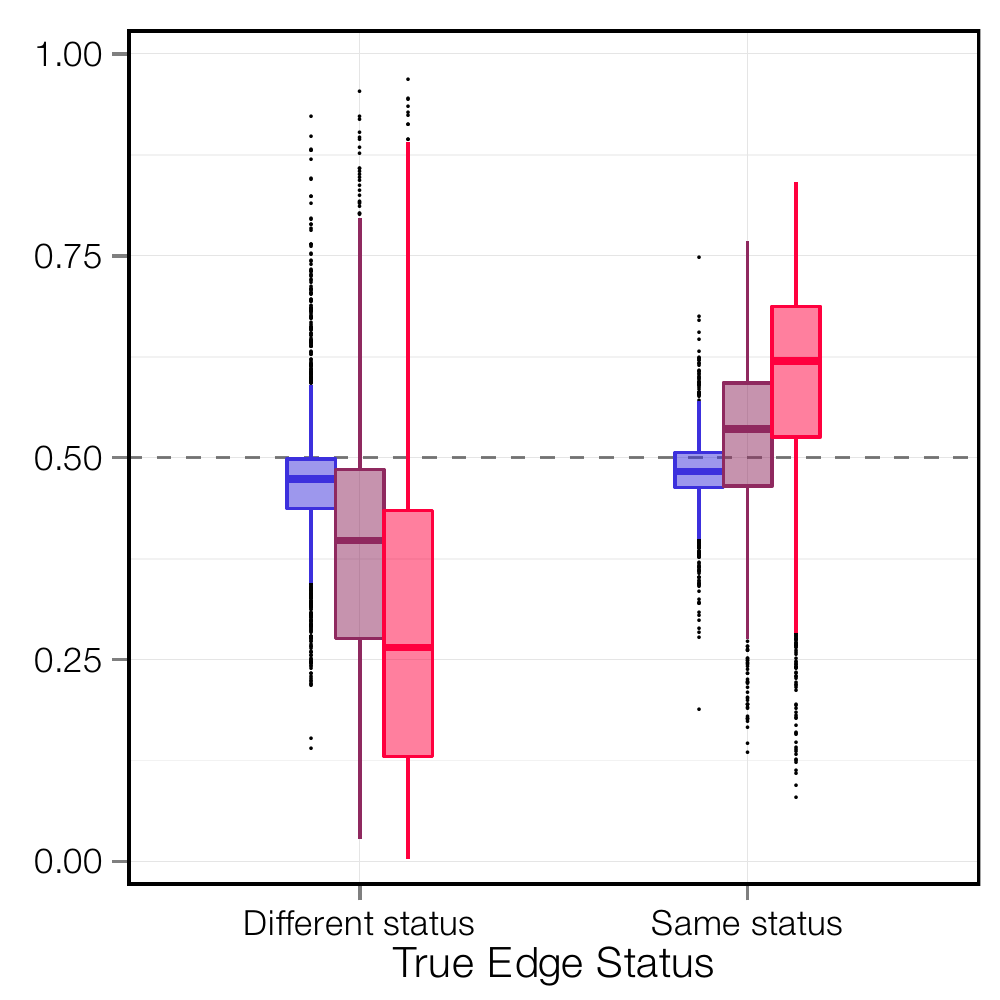}
\includegraphics[width=0.298\textwidth]{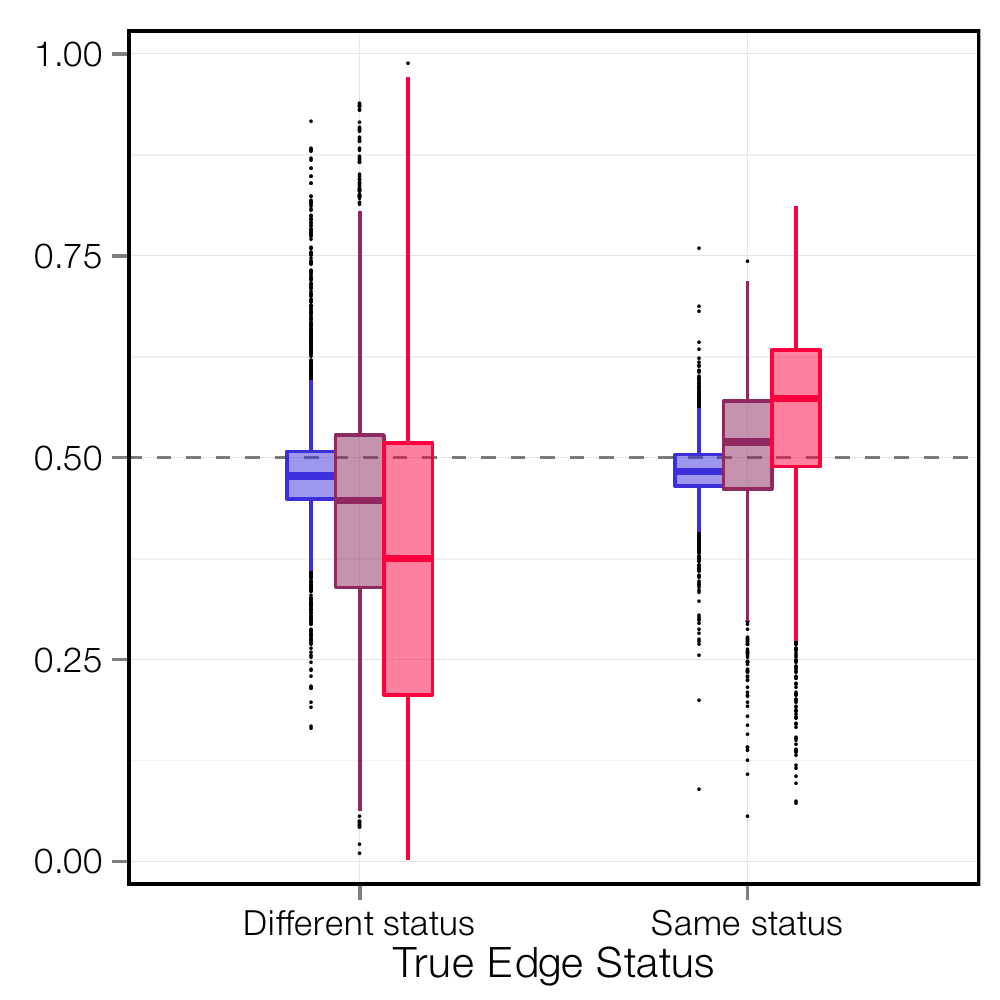}
\includegraphics[width=0.298\textwidth]{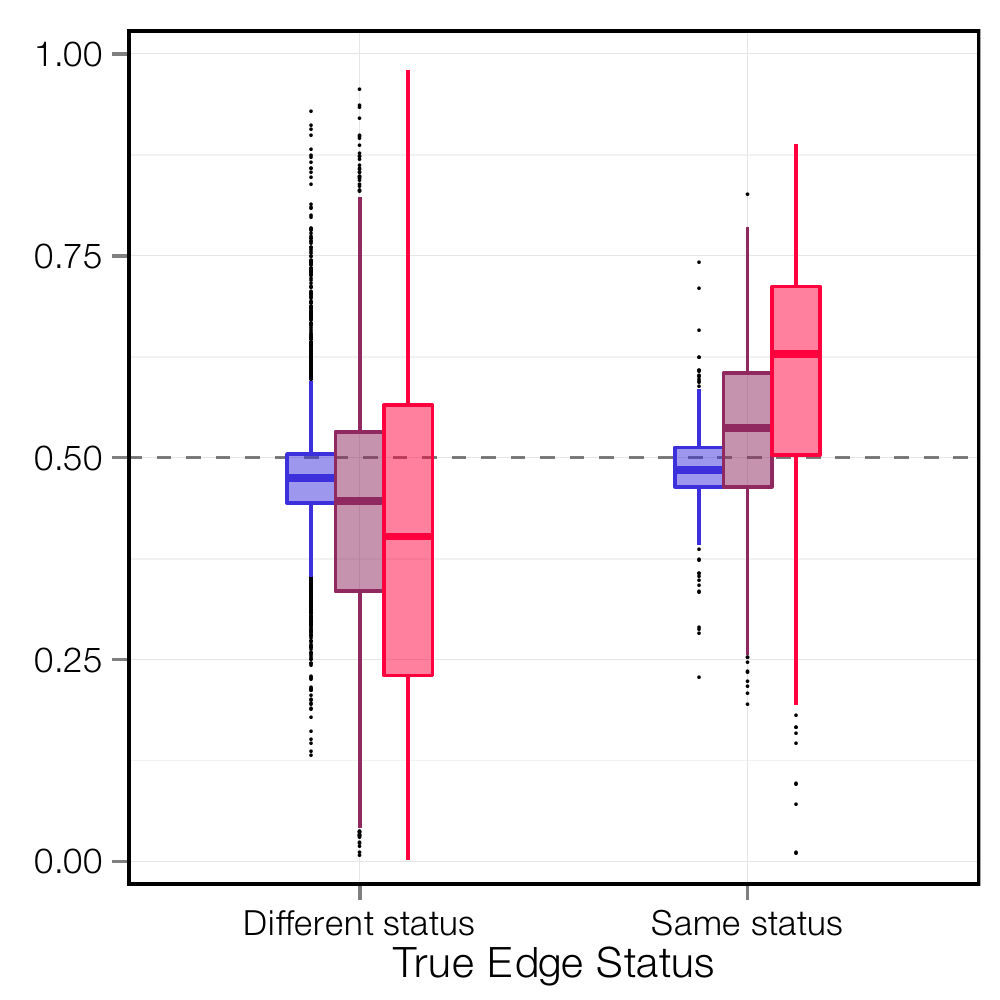}

\caption{Boxplot of the posterior probability for an edge to have the same status throughout the time-series. Edges were separated according to their true status (either identical in all graphs or not). Each boxplot aggregates the results for all edges with a given status and all datasets.}
\label{fig:simu_edge_status}

\end{figure*}

\blue{For this study, we generated time-series of size $N = 70$, $140$ and $210$. We choose segmentations with four segments of lengths $\frac{3}{7} N$, $\frac{1}{7} N$, $\frac{2}{7} N$ and $\frac{1}{7} N$ such that the relative length of each segment is kept identical through all sample sizes.} The number of variables was fixed to $p=10$. To give an idea of the sizes of the discrete sets we are working with, for $N=210$, the cardinalities of the segmentation and tree sets are respectively $|\M_4| \approx 1.5\cdot 10^{6}$ and $|\T| = 10^{8}$, {so the size of the space to be explored is $\approx 1.5\cdot 10^{38}$}. We built three structure scenarios by sampling structures from the uniform distribution on spanning trees, or from an Erd\"{o}s-R\'{e}nyi random graph distribution with connection probability $p_C = 2/p$ or $4/p$. Thus, for each scenario, we got a series $\{\Delta_r\}_{r\in M_N}$ of adjacency matrices describing the structure of the graphical model on all segments. The observations on a segment $r$ were then drawn according to a multivariate Gaussian distribution with mean vector zero and precision matrix $\Lambda_r$ equal to the Laplacian matrix of $\Delta_r$ augmented of 1 on the diagonal, rescaled so that each variable as unit variance. For each sample size and structure series, $100$ datasets were generated.

As described in the introduction of this section, the inference was then performed in the two following models. The first one is the full precision matrix model, without any structure constraint, and is given by
\begin{align}
\{\Lambda_r\}_{r\in m} & ~\textrm{i.i.d.}, & \Lambda_r &\sim \mathcal{W}(\alpha,\phi), \label{eq:simu_W}\\
\{Y_t\}_{t=1}^N & ~\textrm{independent}, & Y^t &\sim \mathcal{N}(\mathbf{0}_p,\Lambda_r), & \forall t\in r. \nonumber
\end{align}
where $\mathcal{W}(\alpha,\phi)$ stands for the Wishart distribution with $\alpha$ degrees of freedom and scale matrix $\phi$. The second one is the corresponding model with tree-structure assumption, as described in Section \ref{subsec:model}, and given by
\begin{align}
\{T_k\}_{k=1}^K & ~\textrm{i.i.d.}, & T_k &\sim \mathcal{U}(\T), \nonumber \\
\{\Lambda_r\}_{r\in m} & ~\textrm{independent}, & \Lambda_r &\sim h\mathcal{W}(\alpha,\phi,T_{\tiny \kappa(r|m)}) \label{eq:simu_W.C} \\
\{Y_t\}_{t=1}^N & ~\textrm{independent}, & Y^t &\sim \mathcal{N}(\mathbf{0}_p,\Lambda_r), ~~~ \forall t\in r, \nonumber
\end{align}
where we let $h\mathcal{W}(\alpha,\phi,T)$ denote the hyper-Wishart distribution based on $\mathcal{W}(\alpha,\phi)$ and with structure $T$ \citep{Schwaller2015}. In both cases, we set $\alpha = p+10$ and $\phi = (\alpha - p - 1)\cdot\mathbf{I}_p$, where $\mathbf{I}_p$ stands for the identity matrix of size $p$. The distribution of $m|K$ is set to the uniform on $\M_K$ and $K$ follows a Poisson distribution with parameter $\gamma = 4$, \blue{truncated to $\llbracket 1;10 \rrbracket$. Results for other prior distributions on $K$ are presented in the supplementary material.}

We emphasize the fact that, when the tree-structured model is considered, the series of precision matrices $\{\Lambda_r\}_{r\in m}$ used to generate the data only belongs to the support of the law in the first structure scenario. The graphs drawn from the Erd\"{o}s-R\'{e}nyi distributions are not trees and therefore cannot induce precision matrices in the support of a tree-structured  hyper-Wishart distribution. On the contrary, the full model obviously allows such precision matrices.

Finally, for the sake of clarity, we limited our study to centered data and null mean models, but one could allow the mean to vary between segments by using a (hyper) normal-Wishart distribution for $(\mu_r,\Lambda_r)$, where $\mu_r$ stands for the mean on segment $r$.

\subsection{Results}

\paragraph{Change-point location} We plotted the posterior probability of a change-point intervening at time $t$, integrated over $K$, as a function of $t$ in the tree-structured and full models (Figure \ref{fig:simu_chg_pt}). In both cases, change-points are hardly retrieved in the high-density Erd\"{o}s-R\'{e}nyi scenario, the inference performing better in the other two low-density scenarios. 
 The standard deviations across samples are lower for the tree-structured model than for the full model. We can also observe a smoother behaviour with respect to time in the tree-structured model.
\blue{Results on simulations with a greater number of segments ($K=10$, displayed in the supplementary material) confirmed these observations. As expected, the shortest segments are hardly detected when the length of the series is small.}
These results seem to show that, when one is interested in retrieved change-point locations,  the tree-structured model that we have presented can be considered in non-tree scenarios without any meaningful drop in performances. 
 
\begin{figure*}[t]
\vspace{-0.3cm}
\centering
\captionsetup[subfigure]{labelformat=empty}
\subfloat[]{\centering
\includegraphics[width=0.6\textwidth]{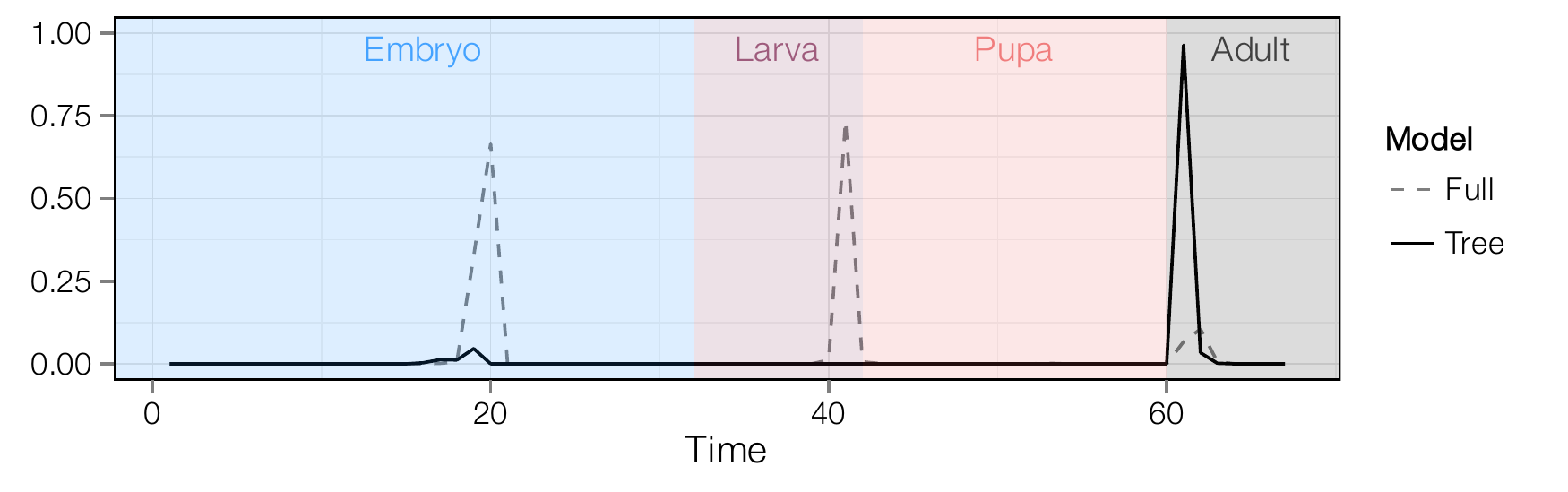}}
\subfloat[]{\centering
\includegraphics[width=0.338\textwidth]{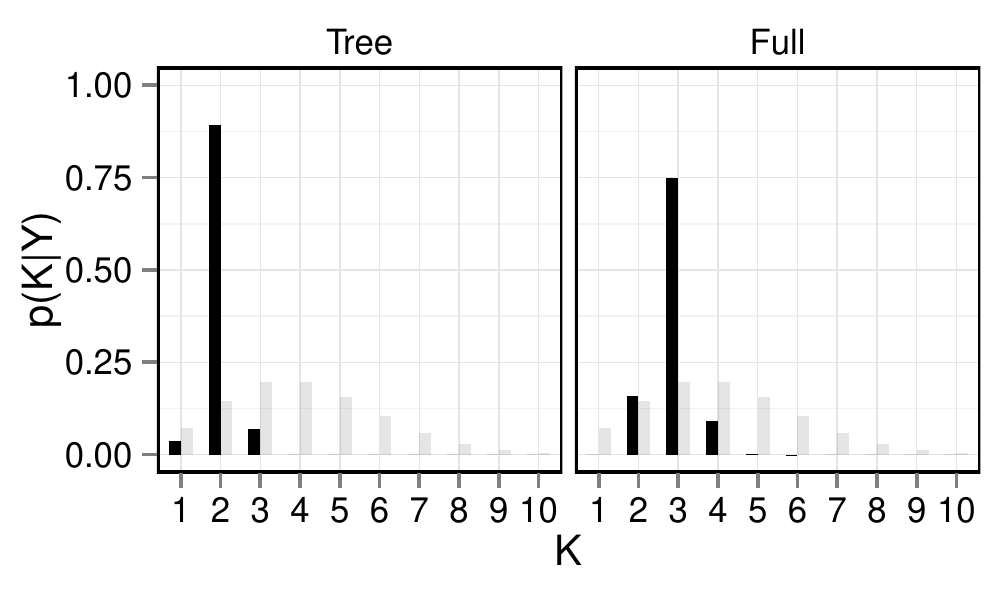}}

\vspace{-0.8cm}

(a) Naive prior

\vspace{-0.5cm}

\subfloat[]{\centering
\includegraphics[width=0.6\textwidth]{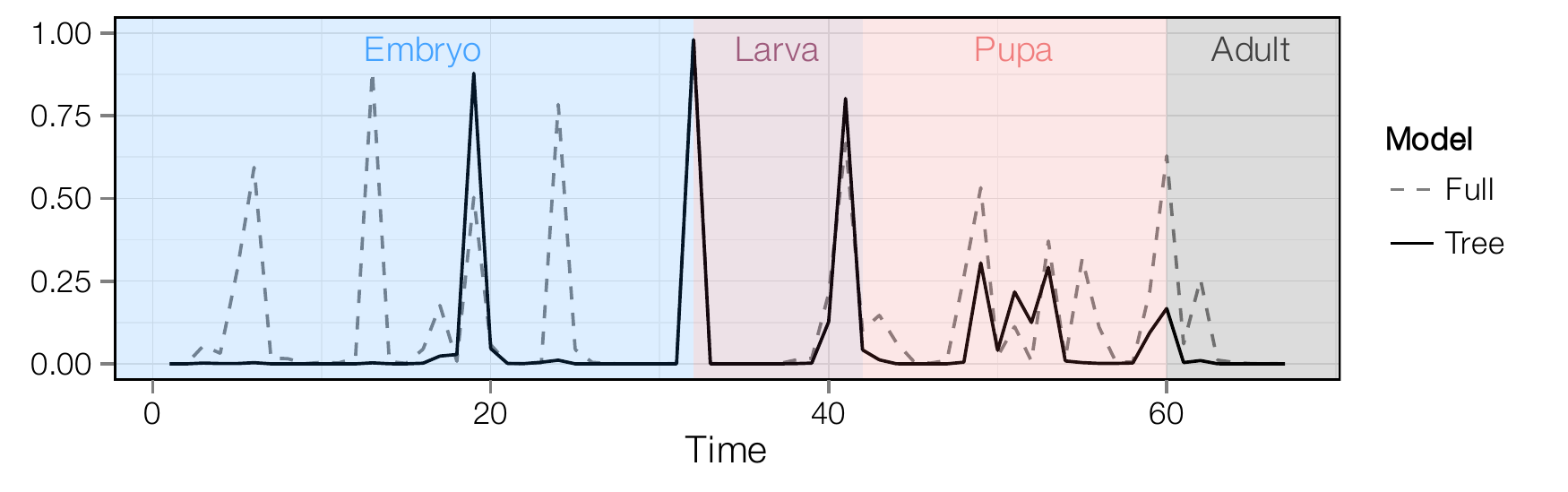}}
\subfloat[]{\centering
\includegraphics[width=0.338\textwidth]{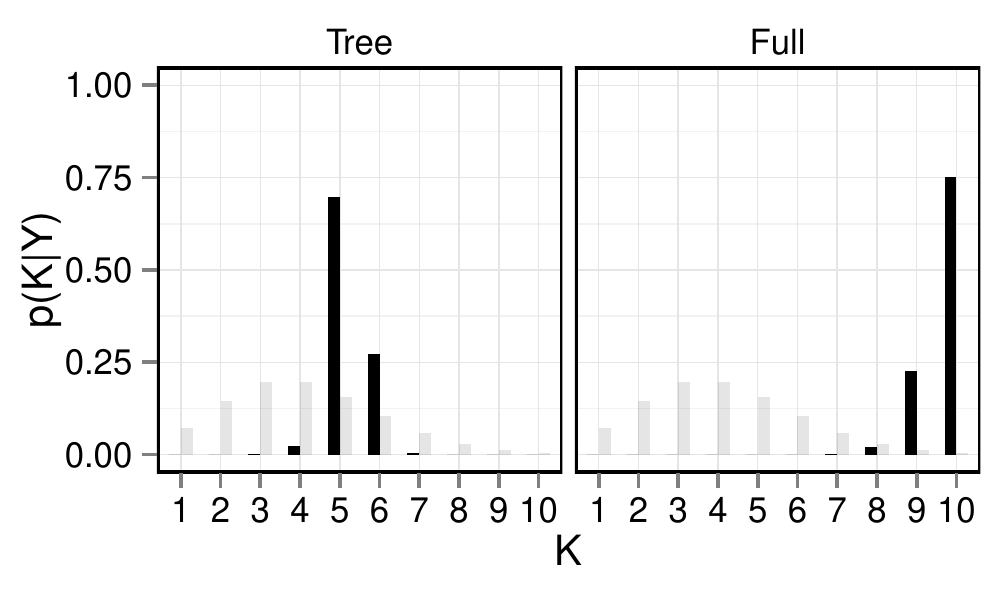}}
\vspace{-0.8cm}

(b) Data-driven prior


\caption{Posterior probability of a change-point occurring at time $t$ as a function of time integrated on $K$ (left) and posterior distribution for $K$ (right) for the full (Full) and tree-structured (Tree) models.}
\label{fig:droso}
\end{figure*}

\begin{figure*}[t]
\centering
\includegraphics[width=\textwidth]{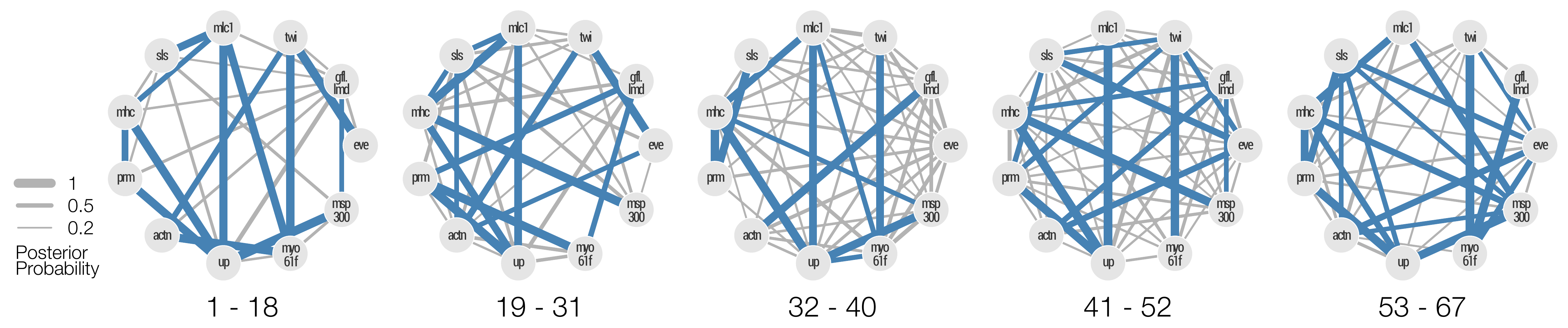}

\caption{Graphical representation of posterior edge probability matrix for each segment of the best segmentation with $5$ segments. The width of an edge is proportional to its posterior probability. Edges with probability higher than 0.5 are coloured in blue. Edges with probability lower than 0.2 were not represented. }
\label{fig:droso_graph}
\end{figure*}

\paragraph{Number of segments} For each sample, we computed $\hat{K}  = \argmax_{K}p(K|y)$ and $K(\hat{m}) = K(\argmax_{m}p(m|y))$. The results are given in Figure \ref{fig:simu_K}. In the full model, the number of segments selected by $\hat{K}$ and $K(\hat{m})$ varies a lot across samples and is usually higher than in the tree model. In the tree-structured model, both $\hat{K}$ and $K(\hat{m})$ tend to slightly underestimate the number of segments, especially in the highly-connected Erd\"{o}s-R\'{e}nyi scenario. They also display a more stable behaviour in the tree model. On small samples, $K(\hat{m})$ seems to achieve better stability.

\paragraph{Posterior edge probability} For $t \in \llbracket 1;N\rrbracket$, we computed the posterior edge probability matrix defined in (\ref{eq:posterior_edge_matrix}) for $K=4$. Figure \ref{fig:simu_aucroc} shows the area under the ROC curve of this matrix against the true adjacency matrix at time $t$. In all scenarios, the structure is better retrieved on long segments. A drop in the accuracy is systematically observed near true change-points. While presenting lower accuracy compared to the other two scenarios, the structure inference in the highly connected scenario still provides meaningful results.

\paragraph{Edge status comparison} 

The posterior probability for an edge to keep the same status throughout time was computed for all edges as explained in Section \ref{sec:edge_status}. We set the prior probability to change status at $\overline{\lambda}= 0.5$ and the prior probabilities to be always present or absent to $\lambda^+ = \lambda^- = 0.25$. We expected edges changing status during the time-series to be given low posterior probabilities. For small samples and across all scenarios, the posterior probability to have the same status remains close to the prior probability $0.5$ for all edges. When samples grow bigger, a small contrast sets up according to the edges effectively changing status or not. We nonetheless observe a large variability across samples and edges, that could be explained by the fact that some configurations are harder to detect than others. An edge only present on a small segment might for instance be considered absent through the whole series.

\section{Applications}
\label{sec:appli}

\subsection{\textit{Drosophila} Life Cycle Microarray Data}

The life-cycle of \textit{Drosophila melanogaster} is punctuated by four main stages of \blue{morphological development}: embryo, larva, pupa and adult. The expression levels of 4028 genes of wild-type \textit{Drosophila} were measured by \cite{Arbeitman2002} at $67$ time-points throughout their life-cycle. We have here restricted our attention to eleven genes involved in wing muscle development and previously studied by \cite{Zhao2006} and \cite{Dondelinger2013}. The expectation was that our approach would find change-points corresponding to the four different stages of \blue{development} observed for \textit{Drosophila melanogaster}.

We used the normal-Wishart version of the model described in the simulation study. When using the naive prior parameters given in Section \ref{sec:simulation}, we obtained poor results (Figure \ref{fig:droso}.a), probably because of the small number of time-points. We noticed that the results could be improved by using data-driven prior specification. \LS{}{We centered the data and set the prior  scale matrix $\phi$ of the normal-Wishart distribution with $\alpha = p +10$ degrees of freedom to $\phi = (\alpha - p - 1)\cdot \mathbf{\Sigma}_y$ where $\mathbf{\Sigma}_y$ stands for the sample covariance matrix}. By doing this, the normal-Wishart distribution that we get has expectancy \LS{}{$(\mathbf{0}_p,\mathbf{\Sigma}_y)$}. We then obtained the results given in Figure \ref{fig:droso}.b. For this prior, we looked closer to the results for $\hat{K} = \arg\max_K p(K|y) = 5$ segments, \textit{i.e. } one more than the number of \blue{development stages}. The best segmentation $\hat{m}_5$ with $5$ segments has change-points at positions $(19,32,41,53)$. The posterior probability of observing a change-point at these locations is quite high (Figure \ref{fig:droso}.b). The larva stage is almost exactly recovered, with a shift of one position for the end of the segment. The embryo stage is divided into two segments and the separation between pupa and adult states is missed, the last segment including both adulthood and part of the pulpa stage. These results are nonetheless encouraging. For each segment $r$ of $\hat{m}_5$, we computed the posterior edge probability matrix given by $(P(\{i,j\}\in E_T | y^{r}) )_{1 \leqslant i,j \leqslant p}$. On each segment, the prior probability for an edge to appear was set to 0.5 with an approach similar to what was done in Section \ref{sec:edge_status}. We give a graphical representation of the results in Figure \ref{fig:droso_graph}. In the first segment, fewer edges have large posterior probabilities. However, this higher contrast in probabilities might just be a consequence of this segment being larger than the others.

{Finally we compared our results with those obtained by \cite{Dondelinger2013} on the same dataset. As for the probability of change-point along time, the results we give in Figure \ref{fig:droso}.b are very similar to those displayed in Figure 12 of this reference. The comparison in terms of inferred networks is more complex as the networks they displayed correspond to the expected stages (embryo, larva, pupa and adult) and not to the one they actually inferred. We found good concordances between the network they inferred for the embryo stage and those that we obtained on segments [1-18] and [19-31] (both in the embryo stage). We also found similarities at the larva stage (which is close to our inferred [32-40] segment). }

\subsection{Functional MRI Data}

Functional magnetic resonance  imaging (fMRI) is commonly used in neuroscience to study the neural basis of perception, cognition, and emotion by detecting changes associated with blood flow. This second application focuses on fMRI data collected by \cite{Cribben2012}. We give a brief description of the experiment but we refer the reader to their article for a more detailed description. Twenty participants were submitted to an anxiety-inducing experiment. Before scanning, participants were told that they would have two minutes to prepare a speech on a subject given to them during scanning. Afterwards, they would have to give their speech in front of expert judges, but they had a ``small chance'' not to be selected. The subject of the speech was given after two minutes of recording. After two minutes of preparation, participants were told that they would not have to give the speech. The recording continued for two minutes afterwards. A series of 215 images at two-second intervals were acquired during the experiment. \cite{Cribben2012} preprocessed the data and determined five regions of interest (ROIs) in the brain on which the signals were averaged. Thus, we have $p=5$ and $N=215$, for $U=20$ participants. We standardised the data across all participants.\\

Each participant can be analysed individually by using the same approach as in the previous application. To analyse all participants together, we make the assumption that the dependence structure between the different ROIs of the brain is the same across participants, while being allowed to vary throughout time. Nonetheless, on a given temporal segment, therefore for a given structure, parameters are independently drawn for each participant, so that the likelihood on a segment $r$ can be written as
\begin{align}
p(y^r) = \sum_{T\in \T} \prod_{u = 1}^U \left[ \int \prod_{t\in r} p(y^{t,u}| \theta_u)p(\theta_u |T)d\theta_u \right] \label{eq:lk_fmri}
\end{align}
where $y^{t,u}$ stands for the vector of observations at time $t$ for participant $u$. The distribution $p(\theta_u|T)$ and $p(y^{t,u}|\theta_u)$ are respectively taken to be normal-Wishart and Gaussian distributions, as in the individual model. In practice, when we tried to perform the inference of the joint model, we were faced with numerical issues, occurring at different levels. The summation over trees was problematic for some segments, especially the largest one. Indeed, we are summing very small quantities and the product over participants in $p(y|T)$ brings us to deal with quantities of the order of machine precision. Moreover, while searching for the best segmentation can be achieved through $\log(A) = [\log(A_{s,t})]_{1 \leqslant s,t \leqslant N+1}$, integrating over segmentations requires the actual computation of matrix $A$. Thus, the exponentiation of the segment log-likelihood matrix leads to other numerical issues.

\begin{figure*}[t]
\subfloat[Posterior change-point probability for five participants \newline with the tree-structured model.]{\label{fig:fmri_chg.pt.a}
\includegraphics[width=0.5\textwidth]{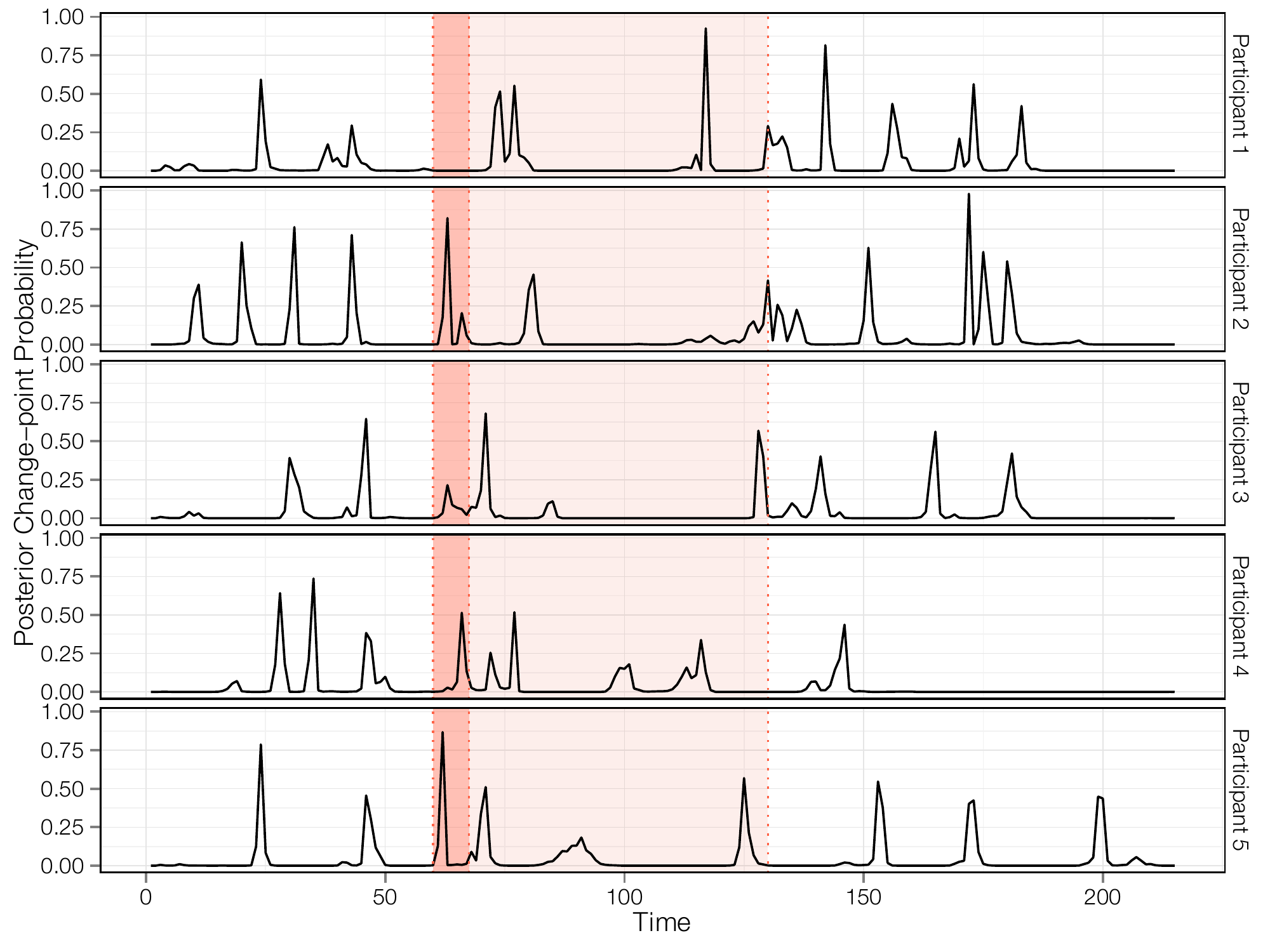}}
\subfloat[Mean and standard deviation of posterior change-point probability across participants with the tree-structured model.]{\label{fig:fmri_chg.pt.b}
\includegraphics[width=0.5\textwidth]{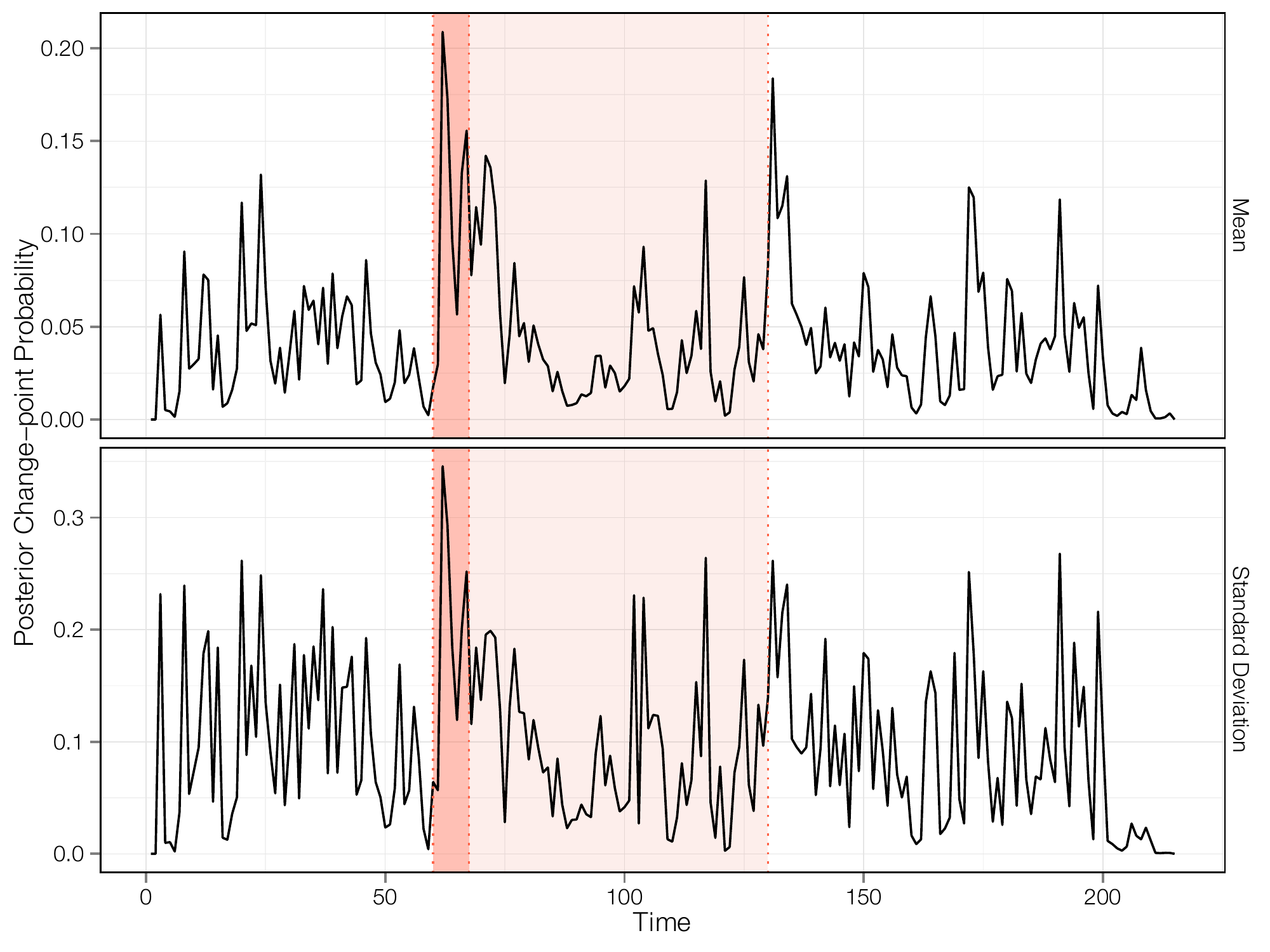}}

\vspace{-0.2cm}

\begin{center}
\subfloat[Posterior change-points probability (left) and posterior distribution on $K$ (right) when participants are jointly considered, with likelihood tempered at a level $\alpha$ for the full (Full) and tree-structured (Tree) models. ]{\label{fig:fmri_chg.pt.c}
\includegraphics[width=0.5625\textwidth]{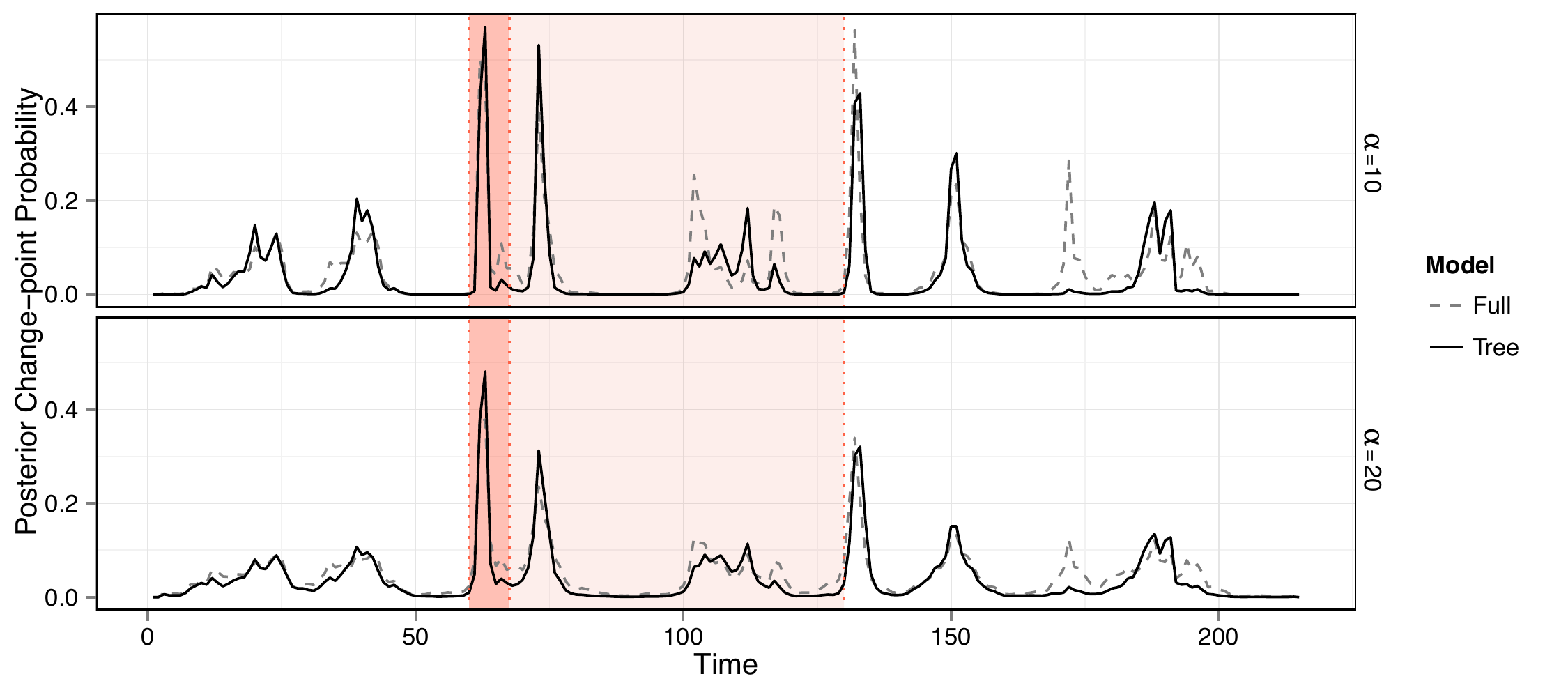}
\includegraphics[width=0.4375\textwidth]{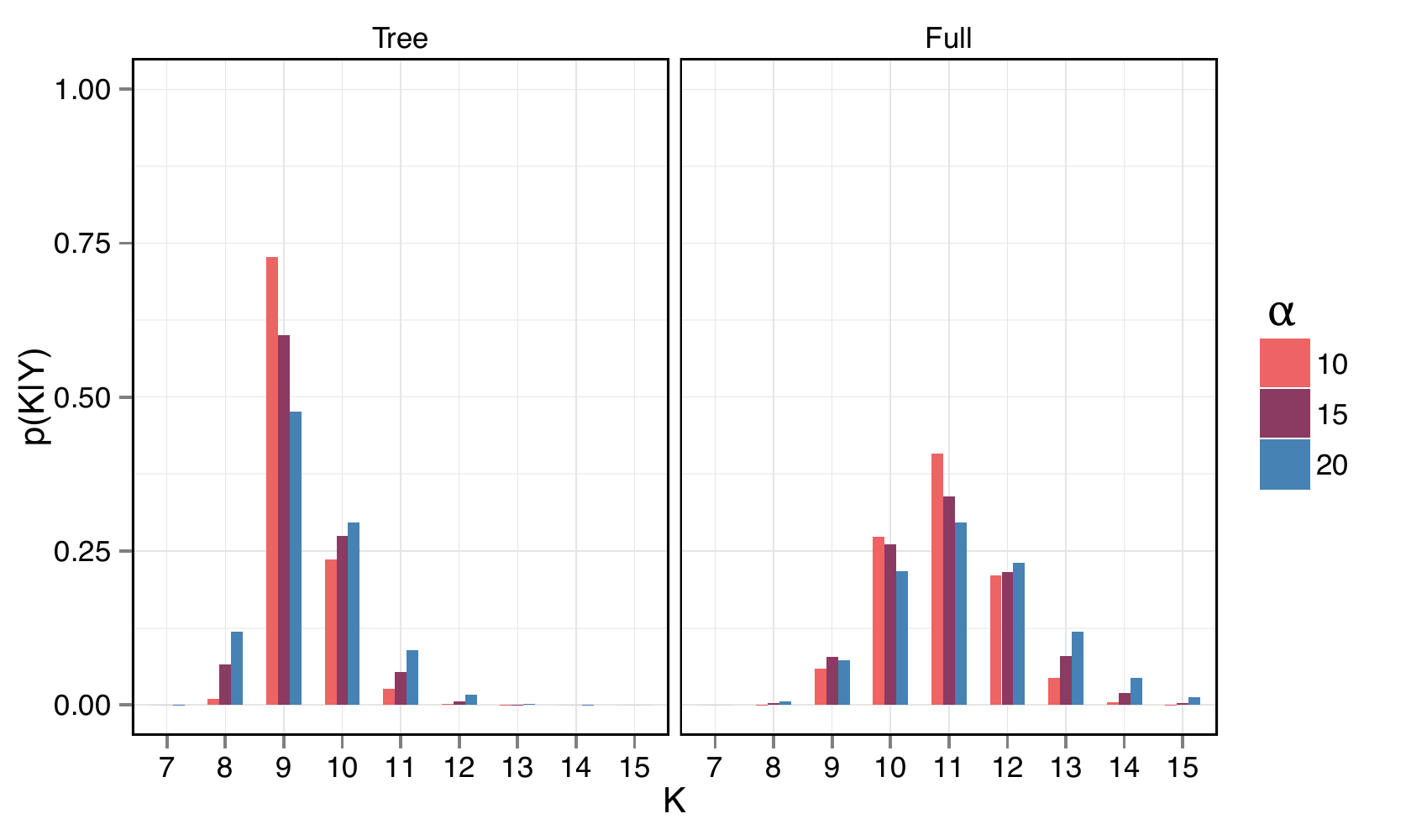}}

\caption{Change-point location for the fMRI data.  During the dark red interval, the subject of the speech was revealed to participants, who prepared their speech during the light red interval. This preparation is ended by a statement saying that they would not have to give the speech.}
\end{center}
\end{figure*}

\begin{figure*}[t]
\vspace{-0.4cm}
\includegraphics[width=\textwidth]{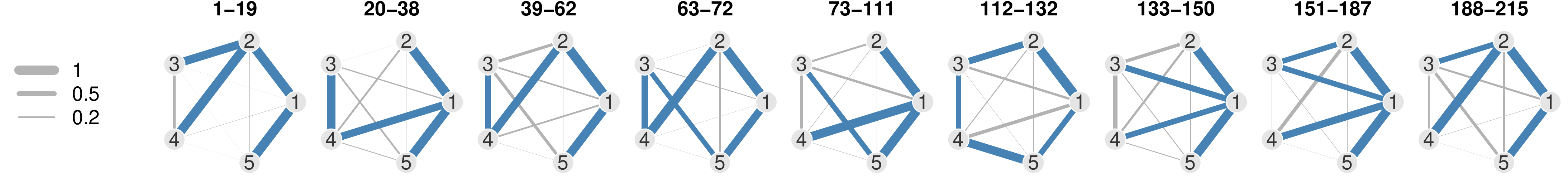}
\caption{Graphical representation of posterior edge probability matrix for each segment of the best segmentation with $K=9$ segments on fMRI data with non-tempered likelihood. The width of an edge is proportional to its posterior probability. Edges with probability higher than 0.5 are coloured in blue.}
\label{fig:fmri_network}
\end{figure*}

Our pragmatic answer to these issues was to considered a tempered version of the likelihood given in (\ref{eq:lk_fmri}):
\begin{align*}
p^*_{\alpha}(y^r) = \sum_{T\in \T} \prod_{u = 1}^U \left[ \int \prod_{t\in r} p(y^{t,u}| \theta_u)p(\theta_u |T)d\theta_u \right]^{1/\alpha}, \label{eq:lk_fmri_tempered}
\end{align*}
with $\alpha > 1$. Tempering the likelihood does not change the mode of the posterior distribution on $m$, if the matrix $a$ giving prior segment weights is tempered similarly. By doing this, we are actually reducing the effective sample size: a very big $\alpha$ would yield a posterior distribution on $m$ close to the prior. \\

Figures \ref{fig:fmri_chg.pt.a} and \ref{fig:fmri_chg.pt.b} sum up the results obtained participant per participant for change-point location. They vary a lot across participants, as shown by the five given examples, as well as the mean and standard deviation curves. The left panel of Figure \ref{fig:fmri_chg.pt.c} shows the posterior probability of observing a change-point when participants are jointly considered with a tree-structured model or with a non-structured model, with likelihood tempered at $\alpha = U/2 = 10$ and $\alpha = U = 20$. For both values of $\alpha$, the profiles are quite similar, with an expected more peaked behaviour for $\alpha = 10$. The strongest peak is observed during the announcement of the speech topic. The right panel of Figure \ref{fig:fmri_chg.pt.c} gives the posterior distribution of $K$ for both models and for different values of $\alpha$. We observe flatter distributions for the full model, with a mode at $11$ segments. In the tree-structured model, $9$ segments are selected. For this value of $K$, we looked at the best segmentation and computed the posterior edge probability matrices for its segments. A graphical representation of the results is given in Figure \ref{fig:fmri_network}. \cite{Cribben2012} retrieved $8$ segments from these data. There is no clear correspondence between our segmentation and theirs. A remark that can nonetheless be made is that, in our case, each change-point is associated with a clear change in the topology of the network. These structure changes are less obvious in \citep{Cribben2012}. This might be a consequence of our model explicitly modelling the structure, thus encouraging change-points to mark out abrupt changes in structure rather than in parameters. 

\section{Discussion}

In this paper, we showed how exact Bayesian inference could be achieved for change-points in the structure of a multivariate time-series with careful specification of prior distributions. Essentially, prior distributions have to factorise over both segments and edges. For the sake of clarity, we assumed that, within a segment $r$, observations $Y^t$ were independent conditionally on $T$ and $\theta$. While convenient and leading to comfortable formulas, this independence assumption is hardly realistic in many applied situations, including those that we have considered here.
Yet, time dependency could be considered within segments, as long as $p(y^r|T)$ still factorises over the edges of $T$. One could for instance consider using the work of \cite{Siracusa2009} to achieve this. Trees would then be used to model the dependences between two consecutive times instead of instantaneous dependences.

The framework that we have described is also convenient for Bayesian model comparison. When one is faced with an alternative in modelling,  Bayes factors between two models are easily obtained, as fully marginal likelihood can be computed exactly and efficiently. For instance, the question of whether changes should be allowed in the mean of a Gaussian distribution or not can be addressed by computing $p(y)$ in both cases and by looking at their ratio. This is by no mean specific to our approach, but exact computation makes it completely straightforward.

The exactness of the inference also creates a comfortable framework to precisely study the effect of the prior distribution on segmentations. Once again, as the inference does not rely on stochastic integration, the impact of prior specification could be evaluated at low cost and in an exact manner.

We finish this discussion by mentioning numerical issues. As explained in Section \ref{sec:appli}, when the number of observations increases, we have to deal with elementary probabilities that differ from several order of magnitudes. Because the summations over the huge spaces of both segmentations and trees are carried out in an exact manner, these quantities have to be added to each other, resulting in numerical errors. Obviously, no naive implementation would work and some of these errors can be avoided with careful and skilful programming. At this stage, this is still not sufficient and the likelihood tempering approach that we propose is not satisfying. Further numerical improvements could be considered such as the systematic ordering of the terms when computing a determinant in a recursive way.

The \textsf{R} code used in the simulations and the applications is available from the authors upon request. A package will soon be available from the Comprehensive \textsf{R} Archive Network. 


\begin{acknowledgements}
The authors thank Ivor Cribben (Alberta School of Business, Canada) for kindly providing the fMRI data. They also thank Sarah Ouadah (AgroParisTech, INRA, Paris, France) for fruitful discussions.
\end{acknowledgements}


\bibliography{biblio_segmentation}

\end{document}